\documentclass[10pt,twocolumn,letterpaper]{article}

\usepackage{wacv}              %

\usepackage{graphicx}
\usepackage{amsmath}
\usepackage{amssymb}
\usepackage{booktabs}

\usepackage{multirow}
\usepackage{multicol}
\usepackage{lineno}
\usepackage{amsmath,amsthm}
\newtheorem{theorem}{Theorem}
\newtheorem{proposition}[theorem]{Proposition}
\usepackage{algorithmic}
\usepackage{booktabs}
\usepackage{caption}
\usepackage[inline]{enumitem}
\usepackage{algorithm}
\usepackage[algo2e]{algorithm2e} 
\usepackage{times}
\usepackage{epsfig}
\usepackage{amssymb}
\usepackage{subcaption}
\usepackage{amsfonts}
\usepackage{array}
\usepackage{textcomp}
\usepackage{makecell}
\usepackage{diagbox}
\usepackage{threeparttable}
\usepackage{mathtools}
\usepackage{eqparbox}
\usepackage{bm}
\usepackage{bbding}
\usepackage{comment}
\usepackage{color}
\usepackage{bbding}
\usepackage{pifont}
\usepackage{wasysym}
\usepackage{cite}
\usepackage{stfloats}
\usepackage{float}

\DeclareMathOperator{\N}{\mathcal{N}}

\usepackage[pagebackref,breaklinks,colorlinks]{hyperref}

\usepackage[capitalize]{cleveref}
\crefname{section}{Sec.}{Secs.}
\Crefname{section}{Section}{Sections}
\Crefname{table}{Table}{Tables}
\crefname{table}{Tab.}{Tabs.}

\def\wacvPaperID{932} %
\def\confName{WACV}
\def\confYear{2025}

\begin{document}

\title{Dual-Schedule Inversion: Training- and Tuning-Free Inversion for\\ Real Image Editing}

\author{Jiancheng Huang$^{*}$
\and Yi Huang$^{*}$ \and Jianzhuang Liu \and Donghao Zhou \and Yifan Liu \and Shifeng Chen$^{\dagger}$\\
Shenzhen Institute of Advanced Technology, Chinese Academy of Sciences\\
{\tt\small jc.huang@siat.ac.cn}}

\makeatletter
\def\blfootnote{\xdef\@thefnmark{}\@footnotetext}
\makeatother

\maketitle
 \blfootnote{\noindent
$^{\dagger}$ Corresponding author.
 $^{*}$ Equal contribution. }

\begin{abstract}
Text-conditional image editing is a practical AIGC task that has recently emerged with great commercial and academic value. For real image editing, most diffusion model-based methods use DDIM Inversion as the first stage before editing. However, DDIM Inversion often results in reconstruction failure, leading to unsatisfactory performance for downstream editing. To address this problem, we first analyze why the reconstruction via DDIM Inversion fails. We then propose a new inversion and sampling method named Dual-Schedule Inversion. We also design a classifier to adaptively combine Dual-Schedule Inversion with different editing methods for user-friendly image editing. Our work can achieve superior reconstruction and editing performance with the following advantages: 1) It can reconstruct real images perfectly without fine-tuning, and its reversibility is guaranteed mathematically. 2) The edited object/scene conforms to the semantics of the text prompt. 3) The unedited parts of the object/scene retain the original identity. 
\end{abstract}

\section{Introduction}\label{sec:intro}

The recent advancement of diffusion models~\cite{song2019generative,song2020denoising,ho2020denoising,ho2022cascaded,nichol2021glide,yu2022parti,ramesh2022hierarchical,rombach2022high} has led to significant progress in text-to-image generation, and enabled a wide range of applications in AI art, film production~\cite{huang2024magicfight,lv2024gpt4motion}, and advertisement design. However, merely generating an image is often not enough. We may also want to edit an existing image using a text prompt. This task is known as text-conditional image editing~\cite{guo2024focus,song2024doubly,li2024zone,brack2024ledits++,nam2024contrastive,huang2024smartedit,liu2024referring,nam2024dreammatcher,huang2024bk,huang2024entwined,huang2024sbcr,huang2023kv}, which can be accomplished using pre-trained text-to-image diffusion models~\cite{nichol2021glide,hertz2022prompt,tumanyan2022plug,parmar2023zero,cao2023masactrl}. Due to its strong practical usage, this task has become popular and holds significant commercial value.

\begin{figure}[t]
    \centering
    \includegraphics[width=1\linewidth]{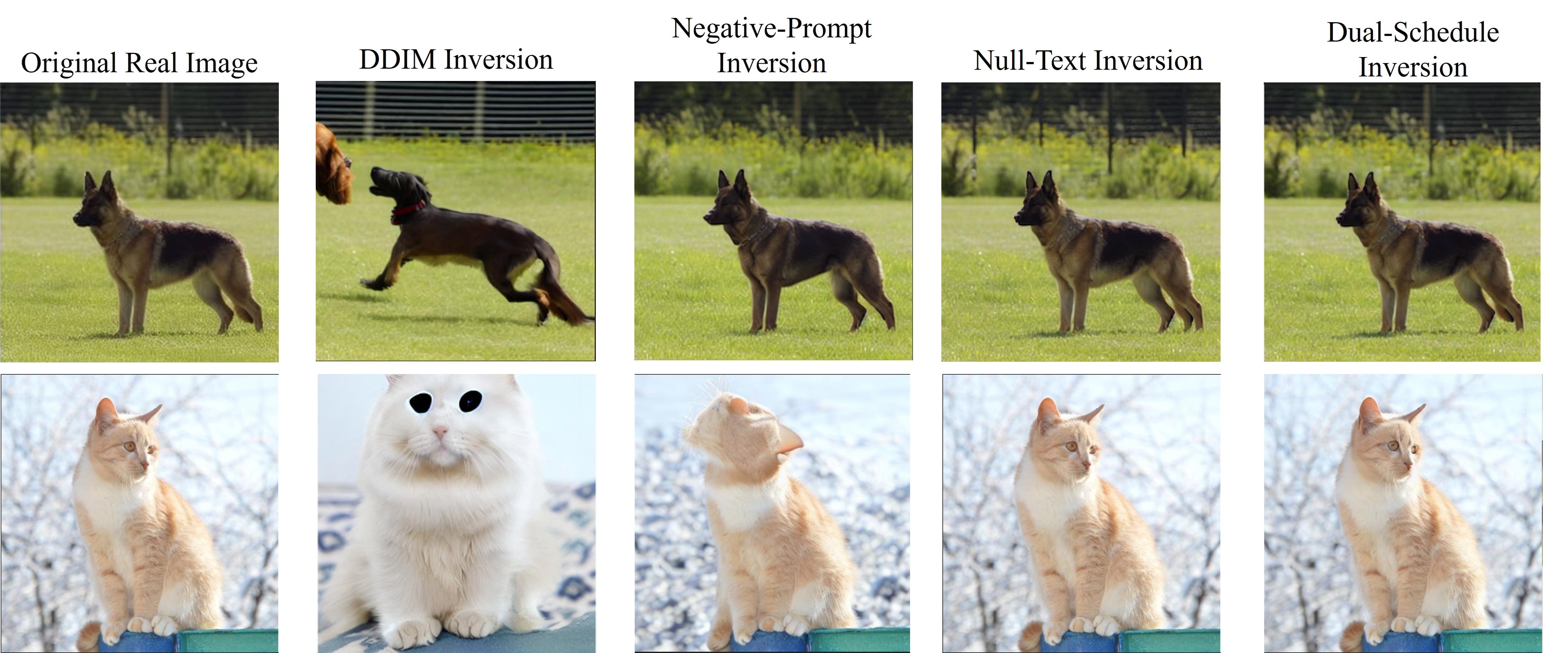}
    \captionof{figure}{
    Examples of reconstruction by different inversion methods \textcolor{black}{with guidance scale $4$}. While Null-Text Inversion requires fine-tuning, the other three methods do not. Dual-Schedule Inversion achieves excellent performance without fine-tuning.
    }
    \vspace{-10pt}
    \label{fig-rec}
\end{figure}

Editing a synthetic image generated by a diffusion model is easier than editing a real image because the initial noisy latent $z_T \sim \mathcal{N}(0,I)$ and the sampling trajectory of this synthetic image are known~\cite{mokady2022null}. For editing a real image via a diffusion model, there are mainly two types of methods, inversion-based and network-tuning-based. In the inversion-based methods~\cite{mokady2022null,miyake2023negative,hertz2022prompt,cao2023masactrl}, the first step is usually to find its initial noisy latent $z_T$, which is called inversion, and then the second step is to edit it with some strategy during the sampling (generation) procedure. The network-tuning-based approaches~\cite{kawar2022imagic,kumari2022customdiffusion,ruiz2022dreambooth,wei2023elite,xiao2023fastcomposer} rely on training or fine-tuning the network for capturing the characteristics of the object/scene in a few images.

Many inversion-based editing methods~\cite{mokady2022null,hertz2022prompt,miyake2023negative,cao2023masactrl} utilize DDIM Inversion~\cite{song2020denoising} for real image reconstruction. DDIM Inversion can obtain an initial noisy latent $z_T$ from a real image $I_s$ by inverting the sampling process, and then use this $z_T$ as the starting point for sampling to obtain a reconstructed image $\hat{I_s}$. With DDIM Inversion, these methods~\cite{mokady2022null,miyake2023negative,hertz2022prompt,cao2023masactrl} first invert the real image into $z_T$, and then generate different contents by modifying the trajectory of $z_t$, $t \in \{T,T-1,...,1\}$, during sampling.
However, DDIM Inversion often fails to reconstruct the original real image~\cite{mokady2022null,miyake2023negative,hertz2022prompt} as shown in Fig.~\ref{fig-rec}. Reconstruction failure makes it difficult to maintain the object/scene consistency between the edited result and the original image.
Several approaches are proposed to address this DDIM Inversion problem. The tuning-based method Null-Text Inversion~\cite{mokady2022null} fine-tunes the unconditional embedding to align the sampling trajectory and the inversion trajectory as closely as possible. However, it requires the extra fine-tuning procedure. The tuning-free methods Negative-Prompt Inversion~\cite{miyake2023negative} and ProxEdit~\cite{han2023improving} argue that the fine-tuning in Null-Text Inversion is unnecessary, but their reconstruction and editing performances are not as good as that of Null-Text Inversion~\cite{miyake2023negative,han2023improving}.

Our work focuses on tuning-free inversion for real image editing without fine-tuning the network or training on a dataset. We first analyze why DDIM Inversion's reconstruction fails mathematically, and then based on this analysis, we propose a tuning-free solution called Dual-Schedule Inversion with inversion and sampling stages. The fundamental difference from DDIM Inversion is that Dual-Schedule Inversion uses two schedules in both stages. We prove that using these two schedules guarantees that the inversion and sampling are perfectly reversible. \textcolor{black}{Moreover, recognizing that different algorithms excel at different editing tasks \cite{huang2024diffusion,huang2024color} (for instance, P2P \cite{hertz2022prompt} is good at object replacement and MasaCtrl \cite{cao2023masactrl} is good at action change), we develop a task classifier that seamlessly integrates our inversion approach with existing algorithms across five widely-used editing tasks. This integration aims to deliver a user-friendly experience by automatically selecting the most appropriate editing algorithm, thus enhancing editing accuracy and contextual preservation efficiently without necessitating manual algorithm selection.}
Our main contributions are summarized as follows. 
\begin{itemize}
\setlength{\itemsep}{0pt}
\setlength{\parsep}{0pt}
\setlength{\parskip}{0pt}
\item We mathematically analyze the reconstruction failure problem in DDIM Inversion and present the reversibility requirement.
\item We propose Dual-Schedule inversion and sampling formulas to ensure the faithful reconstruction of real images, and show how to combine them with other editing methods.
\item \textcolor{black}{For user-friendly editing, we design a classifier to adaptively combine Dual-Schedule \textcolor{black}{Inversion} with different editing methods.}
\item Comprehensive experiments show that Dual-Schedule Inversion achieves superior reconstruction and editing performance incorporated with other editing methods.
\end{itemize}

\section{Related Work}
\subsection{Text-to-Image Generation and Editing}

With the rapid development of diffusion models, text-conditional image generation has experienced an unprecedented explosion~\cite{song2019generative,saharia2022photorealistic,ramesh2022hierarchical,rombach2022high,yu2022parti,huang2023wavedm,zhou2024magictailor}. Diffusion models are a kind of generative models based on nonequilibrium thermodynamics~\cite{sohl2015deep}, which add Gaussian noise associated with a time step to data samples (e.g., images), and train a noise estimation network to predict the noise $\epsilon \sim \mathcal{N}(0,I)$ at different time steps. At the sampling stage after training, the noise estimation network is used to gradually predict and remove the noise until a clean image is finally generated.

Diffusion models for text-conditional image generation align the feature of a textual prompt with image content by a pretrained visual language model, such as CLIP~\cite{li2022clip}. The introduction of textual prompts provides diffusion models with superior and diverse generative capabilities, such as DALL·E 3~\cite{2023dalle3}, LDM~\cite{rombach2022high}, VQ-Diffusion~\cite{gu2022vector}, DALL·E 2~\cite{ramesh2022hierarchical} and GLIDE~\cite{nichol2021glide}. On the basis of image generation, some text-conditional editing methods such as DiffEdit~\cite{couairon2022diffedit}, P2P~\cite{hertz2022prompt}, InstructPix2Pix~\cite{brooks2022instructpix2pix}, PnP~\cite{tumanyan2022plug}, and MasaCtrl~\cite{cao2023masactrl} modify the sampling trajectory of an image for its editing. Besides, many subject-driven image generation methods such as Textual Inversion~\cite{gal2022image}, DreamBooth~\cite{ruiz2022dreambooth}, Custom Diffusion~\cite{kumari2022customdiffusion}, ELITE~\cite{wei2023elite}, and FasterComposer~\cite{xiao2023fastcomposer} can generate different new images of the same subject according to the given subject images, which can also be treated as image editing methods.

\subsection{Diffusion Model Inversion}
Diffusion inversion is to reverse the sampling process, which sequentially corrupts a sample with predicted noises~\cite{song2020denoising}. For real image editing, existing methods mostly utilize DDIM Inversion~\cite{song2020denoising}, which is designed to obtain a deterministic noisy latent $z_T$ from a real image $z_0$, and then uses this $z_T$ as the starting point of the sampling process to modify the trajectory such that a new image is obtained.

However, as mentioned in Sec.~\ref{sec:intro}, reconstruction failure is common in DDIM Inversion, which causes the identity of the unedited parts not to be maintained. In order to solve this problem, some methods~\cite{mokady2022null,ruiz2022dreambooth,gal2022image} choose to fine-tune certain parameters. For instance, Null-Text Inversion~\cite{mokady2022null} aligns inversion and sampling trajectories by fine-tuning the unconditional embedding. Negative-Prompt Inversion~\cite{miyake2023negative} and ProxEdit~\cite{han2023improving} argue that the fine-tuning in Null-Text Inversion~\cite{mokady2022null} is unnecessary and present their tuning-free methods for reconstruction. However, the performances of Negative-Prompt Inversion and ProxEdit are not as good as Null-Text Inversion~\cite{han2023improving,miyake2023negative}. Besides, these methods~\cite{mokady2022null,miyake2023negative,han2023improving} occupy the negative prompt~\cite{ho2021classifier} position for reconstruction, making negative prompts unusable. 

\textcolor{black}{Some recent works propose to modify the inversion and sampling formula. For instance, EDICT~\cite{wallace2023edict} uses affine coupling layers as its inversion and sampling steps. However,~\cite{wallace2023edict} introduces substantial modifications to the sampling formula, which may complicate its integration with other image editing techniques. DDPM Inversion~\cite{huberman2023edit} introduces a DDPM latent noise space for more diversity. Our Dual-Schedule Inversion, on the other hand, emphasizes achieving high-quality image reconstruction by making slight modifications to DDIM inversion and sampling. Thus, editing methods designed for DDIM can also be integrated with our inversion. Moreover, our method's compatibility with various editing techniques makes it adaptable to different editing tasks.}

\textcolor{black}{In Table \ref{tab-type} of our supplementary material, we summarize recent inversion and editing methods. Our Dual-Schedule Inversion is training- and tuning-free, does not occupy the negative prompt position, and is mathematically reversible.}

\section{Background}
\label{sec:background}

\noindent\textbf{Diffusion Model Training.}
The training of a diffusion model starts with a clean sample $z_0$, and a diffusion process is carried out by adding Gaussian noise to $z_0$ as $q(z_{t}\vert z_0)=\N(\sqrt{\alpha_{t}}z_{0},(1-\alpha_{t})I)$,
where $\alpha_t$ is a predefined diffusion schedule, $t\in\{1, 2, ..., T\}$ is the time-step, $z_t$ is the noisy latent, and $z_T \sim \N(0,I)$. In DDPM~\cite{ho2020denoising}, the common setting of $T$ is $1000$. The optimization of the diffusion model is simplified to train a network $\epsilon_{\theta}(z_t,t)$ that predicts the Gaussian noise $\epsilon\sim\N(0,I)$ with this loss: $L_{simple} =  \mathbb{E}_{z_0,t,\epsilon}\Big[\vert\vert{\epsilon} -  \epsilon_{\theta}(z_t,t)\vert\vert^2 \Big]$.

\noindent\textbf{DDIM Sampling and Inversion.}
Given a real image $z_0$, common editing methods~\cite{meng2021sdedit,mokady2022null,cao2023masactrl} first invert this $z_0$ to a $z_T$ by some inversion scheme. Then, they start the sampling process from this $z_T$ and use their editing strategies to generate an edited result $\tilde{z}_0$. Ideally, direct sampling from this $z_T$ without any editing should reconstruct a $\tilde{z}_0$ that is as close to $z_0$ as possible. If this $\tilde{z}_0$ is very different from $z_0$, called reconstruction failure, the corresponding edited image cannot retain the identity of the unedited parts in $z_0$. Therefore, an inversion method that satisfies $\tilde{z}_0 \approx z_0$ is desired. 

We first analyze the commonly used DDIM sampling and inversion, whose sampling equation is as follows:
\begin{equation}\label{ddim sampling}
\small
    z_{t-1} =
\sqrt {\alpha_{t-1}} \frac{z_t - \sqrt{1-\alpha_t}   \epsilon_\theta(z_t,t) }{\sqrt {\alpha_t}}
+ \sqrt{1-\alpha_{t-1}}   \epsilon_\theta(z_t,t),
\end{equation}
which can be rewritten as:
\begin{equation}\label{ddim invert ideal}
\small
z_t =   \sqrt{\alpha_t}  \frac{z_{t-1} - \sqrt{1-\alpha_{t-1}}   \epsilon_\theta(z_{t},t) }{\sqrt {\alpha_{t-1}}}+ \sqrt{1-\alpha_{t}}   \epsilon_\theta(z_{t},t).
\end{equation}
Eq.~\ref{ddim invert ideal} seems to be used as a perfect inversion from $z_{t-1}$ to $z_t$. However, the problem is that $z_t$ is unknown and used as the input to the network $\epsilon_\theta(z_{t},t)$. Thus, DDIM Inversion~\cite{song2020denoising} and several methods~\cite{wallace2023edict, pan2023effective, hong2023exact} assume that $z_{t-1} \approx z_t$, and replace $z_t$ on the right hand side of Eq.~\ref{ddim invert ideal} with $z_{t-1}$, resulting in the following approximation:
\begin{equation}\label{ddim invert}
\small
\begin{aligned}
   z_t = \sqrt{\alpha_t}  \frac{z_{t-1} - \sqrt{1-\alpha_{t-1}}   \epsilon_\theta(z_{t-1},t) }{\sqrt {\alpha_{t-1}}}+\sqrt{1-\alpha_{t}}   \epsilon_\theta(z_{t-1},t).
\end{aligned}
\end{equation}

\noindent\textbf{Text Condition and Classifier-Free Guidance.}
Text-conditional diffusion models aim to generate a result from a random noise $z_T$ with a text prompt $P$. During the sampling process at inference, the noise estimation network $\epsilon_\theta(z_t,t,C)$ is used to predict the noise $\epsilon$, where $C = \psi(P)$ is the text embedding.
The noise in $z_t$ is gradually removed for $T$ steps until $z_0$ is obtained. 

In text-conditional image generation, it is necessary to give enough textual control and influence over the generation. Ho et al. \cite{ho2021classifier} propose classifier-free guidance where conditional and unconditional predictions are combined. Specifically, let $\varnothing = \psi(``")$\footnote[1]{The position for $\varnothing$ is often used by negative prompts such that some attributes do not appear in the generated image.} be the null text embedding and $w$ be the guidance scale. Then the classifier-free guidance prediction is defined by:
\begin{equation}\label{cfg}
\epsilon_\theta(z_t,t,C,\varnothing) = w   \epsilon_\theta(z_t,t,C) + (1-w) \   \epsilon_\theta(z_t,t,\varnothing),
\end{equation}
where $\epsilon_\theta(z_t,t,C,\varnothing)$ is used to replace $\epsilon_\theta(z_t,t)$ in the sampling Eq.~\ref{ddim sampling}, and $w$ is usually in $[1, 7.5]$ \cite{rombach2022high,saharia2022photorealistic}. The higher $w$ means the stronger control by the text.

\noindent\textbf{Image Editing.}
Given a real image $I_s$ and a related text prompt $P_s$, the goal is to generate a new image $I_t$ with the target prompt $P_t$ using a pretrained diffusion model such that $I_t$ aligns with $P_t$. This task is challenging, particularly on real images, as most image editing methods struggle to edit real images while preserving good reconstruction performance~\cite{hertz2022prompt,mokady2022null,tumanyan2022plug}. A popular way used by many editing methods~\cite{cao2023masactrl,hertz2022prompt,miyake2023negative,tumanyan2022plug} is: 1) using DDIM Inversion on the original image $I_s$ to get its $z_T$, and 2) using $P_t$ and this $z_T$ as the starting point to generate $I_t$. However, as shown in Fig.~\ref{fig-rec}, DDIM Inversion often fails to reconstruct the original image, making the identity of the object/scene in $I_t$ changed.

\section{Method}

We first analyze why the reconstruction of a real image via DDIM Inversion fails in Sec.~\ref{sec:fail}. Then, in Sec.~\ref{sec:Dual-Scheduler}, we propose our Dual-Schedule Inversion for inversion and sampling, and show that it mathematically guarantees perfect reconstruction. \textcolor{black}{Finally, we introduce five categories of editing tasks and our task classifier that integrates our inversion approach with existing editing methods.}

\subsection{Irreversibility of DDIM Inversion}\label{sec:fail}

In DDIM sampling~\cite{song2020denoising}, to speed up the time-consuming sampling procedure, $[t_0, t_0+s,t_0+2s,... , t_0+T^{\prime}s]$, instead of $[1, 2, ..., T]$, is chosen where $s$ is the interval and $T^{\prime}<T$. For example, if $t_0=1$, $s=20$, and $T^{\prime}=50$, then the schedule is $[1, 21, 41, ..., 961, 981]$; if $t_0=10$, $s=20$, and $T^{\prime}=50$, then the schedule is $[10, 30, 50, ..., 990]$. In the rest of this paper, as shown in Fig.~\ref{fig-irre}, we use $[t_0, t_0+s, ..., t_0+T^{\prime}s]$ to denote the sampling schedule, and thus the noisy latents during sampling are $\tilde z_{t_0+T^{\prime}s}, ..., \tilde z_{t_0+s}, \tilde z_{t_0}$. In this case, Eq.~\ref{ddim sampling} is rewritten as:
\begin{subequations}
\begin{equation}\label{ddim sample sim}
    \tilde z_{t-s} =  a_{(t\rightarrow t-s)}\tilde z_{t} + b_{(t\rightarrow t-s)}  \epsilon_\theta(\tilde z_{t},t),
\end{equation}
\begin{equation}\label{ba}
\begin{aligned}
  &a_{(t\rightarrow t-s)}=\sqrt{\alpha_{t-s}/\alpha_{t}},\\
      &b_{(t\rightarrow t-s)}= \sqrt{1-\alpha_{t-s}}-\frac{\sqrt{1-\alpha_t}\sqrt{\alpha_{t-s}}}{\sqrt {\alpha_t}},
\end{aligned}
\end{equation}
\end{subequations}
where $b_{(t\rightarrow t-s)}$ and $a_{(t\rightarrow t-s)}$ are constant coefficients from $t$ to $t-s$. Besides, we use $\bar{z}_{t_0}, \bar{z}_{t_0+s}, ..., \bar{z}_{t_0+T^{\prime}s}$ to denote the noisy latents during inversion. Again, we rewrite the approximation Eq. \ref{ddim invert} (DDIM inversion) as:
\begin{subequations}
\begin{equation}\label{ddim invert sim}
    \quad\quad\quad\bar z_t =  a_{(t-s\rightarrow t)}\bar z_{t-s} + b_{(t-s\rightarrow t)}  \epsilon_\theta(\bar z_{t-s},t),
\end{equation}
\begin{equation}\label{ba2}
  \begin{aligned}
  &a_{(t-s\rightarrow t)}=\sqrt{\alpha_{t}/\alpha_{t-s}},\quad\quad\quad\\
      &b_{(t-s\rightarrow t)}= \sqrt{1-\alpha_{t}}-\frac{\sqrt{1-\alpha_{t-s}}\sqrt{\alpha_{t}}}{\sqrt {\alpha_{t-s}}},\quad\quad\quad
  \end{aligned}
\end{equation}
\end{subequations}
where $a_{(t-s\rightarrow t)}$ and $b_{(t-s\rightarrow t)}$ are another two constant coefficients from $t-s$ to $t$. Furthermore, Eq.~\ref{ddim invert sim} can be derived as:
\begin{equation}\label{ddim relate sim}
    \bar z_{t-s} =  a_{(t\rightarrow t-s)}\bar z_{t} + b_{(t\rightarrow t-s)}  \epsilon_\theta(\bar z_{t-s},t).
\end{equation}
Comparing Eq.~\ref{ddim sample sim} and Eq.~\ref{ddim relate sim}, if we want to have the exact inversion $\bar z_{t-s} = \tilde z_{t-s}$, it is necessary that $\bar z_{t} = \tilde z_{t}$ and $\tilde z_t=\bar z_{t-s}$, which lead to $\bar z_{t_0+T^{\prime}s}=\bar z_{0}$, meaning that there is no any diffusion. Therefore, DDIM Inversion is mathematically irreversible. In the next section, we propose a solution.

\subsection{Dual-Schedule Inversion}\label{sec:Dual-Scheduler} 

\noindent\textbf{Reversibility Requirement.} Based on Eq.~\ref{ddim sample sim} and Eq.~\ref{ddim invert sim}, we define the following sampling and inversion formulas:
\begin{equation}\label{what sample sim}
    \tilde z_{t-s} =  a_{(t\rightarrow t-s)}\tilde z_{t} + b_{(t\rightarrow t-s)}  \epsilon_\theta(\tilde z_{\tau},\tau),
\end{equation}
\begin{equation}\label{what invert sim}
    \bar z_t =  a_{(t-s\rightarrow t)}\bar z_{t-s} + b_{(t-s\rightarrow t)}  \epsilon_\theta(\bar z_{\tau},\tau),
\end{equation}
where $a_{(t\rightarrow t-s)}$, $b_{(t\rightarrow t-s)}$, $a_{(t-s\rightarrow t)}$ and $b_{(t-s\rightarrow t)}$ are computed via Eq.~\ref{ba} and Eq.~\ref{ba2}.
We call $\tilde z_{\tau}$ and $\bar z_{\tau}$ satisfying $\tilde z_{\tau}=\bar z_{\tau}$ the \textit{auxiliary latents}. 
\begin{figure}[t]
    \centering
\includegraphics[width=0.9\linewidth]{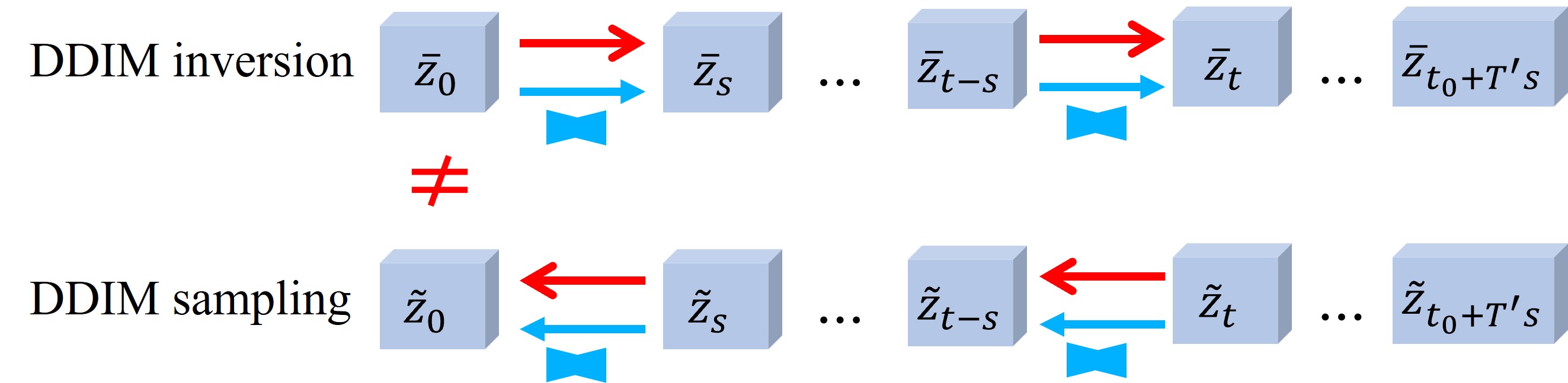}
    \caption{DDIM Inversion is irreversible. }
    \label{fig-irre}
    \vspace{-10pt}
\end{figure}

Similar to Eq.~\ref{ddim relate sim}, Eq.~\ref{what invert sim} can be rewritten as:
\begin{equation}\label{what related sim}
    \bar z_{t-s} =  a_{(t\rightarrow t-s)}\bar z_{t} + b_{(t\rightarrow t-s)}  \epsilon_\theta(\bar z_{\tau},\tau).
\end{equation}
Comparing Eq.~\ref{what sample sim} and Eq.~\ref{what related sim}, we can see that $\tilde z_{t-s}= \bar z_{t-s}$ as long as $\tilde z_{t}= \bar z_{t}$. Therefore, satisfying $\tilde z_{\tau}=\bar z_{\tau}$ and $\tilde z_{t}= \bar z_{t}$ guarantees the perfect reversibility ($\tilde z_{t-s}= \bar z_{t-s}$) in Eq.~\ref{what sample sim} and Eq.~\ref{what invert sim}. Since $\tilde z_{t_0+T^{\prime}s}= \bar z_{t_0+T^{\prime}s}$, now the problem becomes how to find $\tilde z_{\tau}$ and $\bar z_{\tau}$ such that $\tilde z_{\tau}=\bar z_{\tau}$.

\noindent \textbf{Primary and Auxiliary Schedules.}
To obtain the auxiliary latents, we design two schedules, as shown in Fig.~\ref{fig:main_arch}(a), where the upper and lower parts are called the \textit{primary schedule} and the \textit{auxiliary schedule}, respectively. 

The primary time schedule is, e.g., $[1,21,41,...,981]$ with $T^{\prime}=50$. In these $50$ steps, we use $\{z^{p}_t\}$ to denote the \textit{primary latents} $[z^{p}_{1},z^{p}_{21}, z^{p}_{41}, ..., z^{p}_{981}]$. Given these time-steps of the primary schedule, we also design an auxiliary schedule $[10,30,50,...,970]$. This auxiliary schedule is for obtaining the auxiliary latents $z_\tau=\tilde z_\tau = \bar z_\tau$, where $\tau \in \{10,30,50,...,970\}$. For example, between time-steps $21$ and $41$ in the primary schedule, $\tau$ is chosen as the midpoint $30$. We use $\{z^{a}_t\}$ to denote the auxiliary latents $[z^{a}_{10},z^{a}_{30}, z^{a}_{50}, ..., z^{a}_{970}]$. 

\noindent \textbf{Dual-Schedule Inversion.}
As shown in Fig.~\ref{fig:main_arch}(a), we design the iterative inversion of the primary schedule as follows where $\bar z^p_t$ and $\bar z^a_t$ denote the latents during inversion:
\begin{equation}\label{invert odd}
\begin{aligned}
\bar z^{p}_{t} = &a_{(t-20\rightarrow t)}\bar z^p_{t-20} + b_{(t-20\rightarrow t)}  \epsilon_\theta(\bar z^a_{t-11},t-11),
\end{aligned}
\end{equation}
where $t=[21, ..., 961, 981]$. Eq.~\ref{invert odd} is one implementation of Eq.~\ref{what invert sim} and $\bar z^a_{t-11}$ is the auxiliary latent from $\bar z^p_{t-20}$ to $\bar z^p_{t}$. 

Also, for the iterative inversion of the auxiliary schedule, we have the inversion formula:
\begin{equation}\label{invert even}
\begin{aligned}
    \bar z^{a}_{t} = &a_{(t-20\rightarrow t)}\bar z^a_{t-20} + b_{(t-20\rightarrow t)}  \epsilon_\theta(\bar z^p_{t-9},t-9),
\end{aligned}
\end{equation}
where $t=[30, ..., 950, 970]$. In Eq.~\ref{invert odd} and Eq.~\ref{invert even}, $\bar z_1^{p}$ and $\bar z_{10}^{a}$ are obtained by the original forward process of DDIM since they are the starting of the schedule. This calculation does not affect our reversibility, which is proved in the supplementary material.

\noindent \textbf{Dual-Schedule Sampling.} To reconstruct $I_s$ perfectly, as shown in Fig.~\ref{fig:main_arch}(b), we devise a sampling formula Eq.~\ref{sampling odd} for the primary schedule, where $\tilde z^p_t$ and $\tilde z^a_t$ denote the latents during sampling:
\begin{equation}\label{sampling odd}
\begin{aligned}
\tilde z^p_{t-20} = & a_{(t\rightarrow t-20)}\tilde z^p_{t} + b_{(t\rightarrow t-20)}  \epsilon_\theta(\tilde z^a_{t-11},t-11),
\end{aligned}
\end{equation}
where $t=[981, 961, ..., 21]$. Eq.\ref{sampling odd} is one implementation of Eq.~\ref{what sample sim}. Comparing Eq.~\ref{sampling odd} and Eq.~\ref{invert odd} with Eq.~\ref{what sample sim} and Eq.~\ref{what invert sim}, since $\tilde z^p_{t_0+T^{\prime}s}= \bar z^p_{t_0+T^{\prime}s}$, it is obvious that we can guarantee $\tilde z^p_{t-20}=\bar z^p_{t-20}$ as long as $\tilde z^a_{t-11}=\bar z^a_{t-11}$. For obtaining $\tilde z^a_{t-11}=\bar z^a_{t-11}$, given $\tilde z^a_{t_0+T^{\prime}s}= \bar z^a_{t_0+T^{\prime}s}$, we design the sampling formula for the auxiliary schedule:
\begin{equation}\label{sampling even}
\begin{aligned}
\tilde z^a_{t-20} = & a_{(t\rightarrow t-20)}\tilde z^a_{t} + b_{(t\rightarrow t-20)}  \epsilon_\theta(\tilde z^p_{t-9},t-9),
\end{aligned}
\end{equation}
where $t=[970, 950, ..., 30]$. Eq.~\ref{invert even} and Eq.~\ref{sampling even} are also one implementation of Eq.~\ref{what invert sim} and Eq.~\ref{what sample sim}, respectively.

\noindent\textbf{Reversibility.} As stated in the beginning of this section, the reversibility requirement is $\tilde z_{\tau}=\bar z_{\tau}$ and $\tilde z_{t}=\bar z_{t}$. 
 
 Given $\tilde z^p_{981}=\bar z^p_{981}$ and $\tilde z^a_{970}=\bar z^a_{970}$, we have $\tilde z^p_{961}=\bar z^p_{961}$ according to Eq.~\ref{invert odd} and Eq.~\ref{sampling odd}. Then, given $\tilde z^p_{961}=\bar z^p_{961}$ and $\tilde z^a_{970}=\bar z^a_{970}$, we have $\tilde z^a_{950}=\bar z^a_{950}$ according to Eq.~\ref{invert even} and Eq.~\ref{sampling even}. Repeating this process, our Dual-Schedule Inversion can finally guarantee perfect reversibility $\tilde z^p_{1} = \bar z^p_{1}$, as illustrated in Fig.~\ref{fig:main_arch} (a) and (b).

\begin{figure*}[t]
    \centering
    \includegraphics[width=0.99\linewidth]{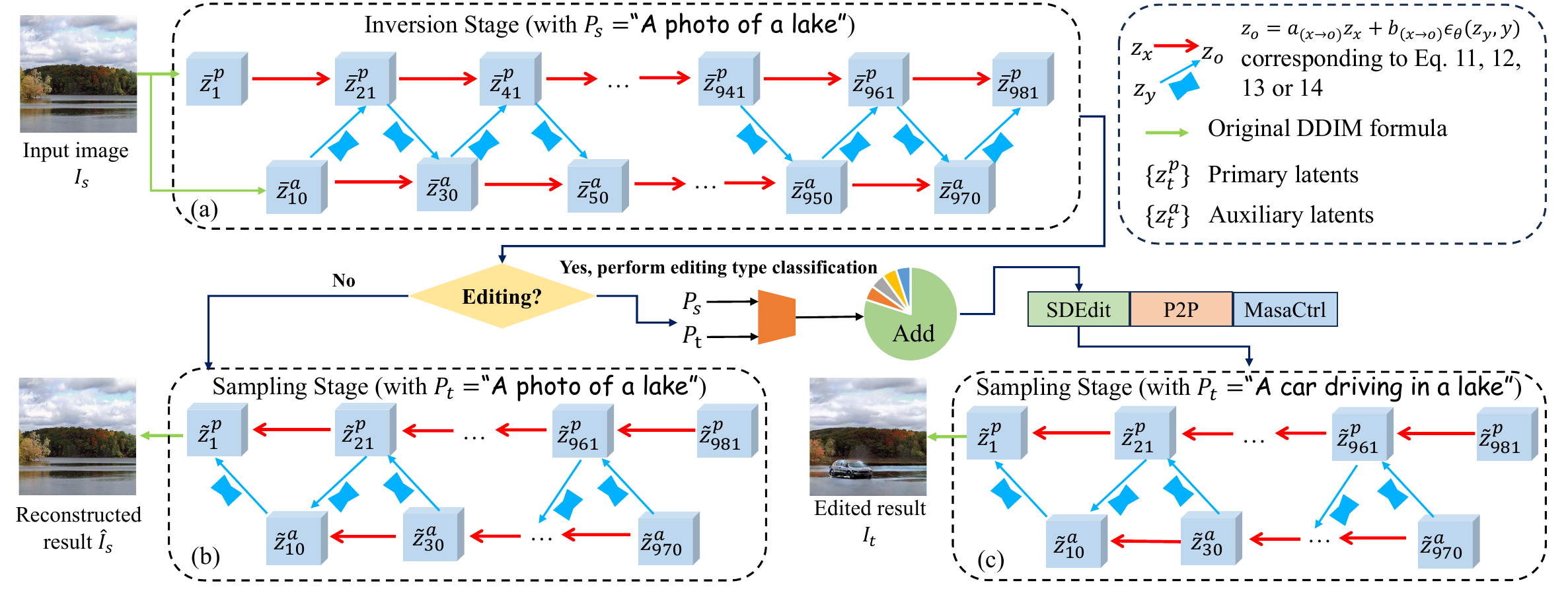}
    \caption{Pipeline of Dual-Schedule Inversion, which is divided into inversion and sampling stages. The inversion stage (a) is for getting $\bar z^p_{981}$ and $\bar z^a_{970}$. The sampling stage is used for reconstruction (b) or editing (c) depending on the target prompt. Both stages have two time schedules where two specific schedules $[1,21,41,...,981]$ and $[10,30,50,...,970]$ are used for example.} 
    \vspace{-15pt}
    \label{fig:main_arch}
\end{figure*}

\subsection{Dual-Schedule Inversion for Editing}\label{sec:Dual-edit}

 After the above design of the Dual-Schedule Inversion for reconstruction, it can be easily combined with common editing methods for editing. For example, a direct editing way is to use a different target prompt in Dual-Schedule Inversion's sampling (see Fig.~\ref{fig:main_arch}(c)). We can also replace DDIM Inversion with Dual-Schedule Inversion in the recent editing methods~\cite{hertz2022prompt,meng2021sdedit,cao2023masactrl}. 

 \noindent \textbf{Five Categories of Editing Tasks.}
We classify text-conditional image editing into five categories: 1) \textit{Object Replacement}: Replacing an object in the image with another object; 2) \textit{Action Editing}: Preserving the background and the object identity but altering his/her/its action; 3) \textit{Scene Editing}: Keeping the objects in the image unchanged while transforming the scene; 4) \textit{New Object Creation}: Adding a new object to the image while keeping the unedited part of the image unchanged; 5) \textit{Style Editing}: Modifying the artistic style of the image. Our Dual-Schedule Inversion is a universal inversion method that can be adapted to all these editing categories by combining it with other image editing techniques, such as SDEdit~\cite{meng2021sdedit}, P2P~\cite{hertz2022prompt}, and MasaCtrl~\cite{cao2023masactrl}.

\noindent \textbf{Automatic Determination of Editing Tasks based on Descriptions.}
\textcolor{black}{
Given the diverse requirements of the five editing tasks, it is imperative to tailor distinct editing methods to different tasks, because different methods are good at different tasks.
For the object replacement and scene editing tasks, we combine our Dual-Schedule Inversion with P2P~\cite{hertz2022prompt}; for the action editing task, we combine it with MasaCtrl~\cite{cao2023masactrl}; for new object adding and style editing, we combine it with SDEdit~\cite{meng2021sdedit}. However, making users identify editing types and choose algorithms is overly complex. To automate and seamlessly cater to users' needs across various editing tasks, we propose an approach capable of automatically determining the editing task.}

\textcolor{black}{
Specifically, with in-context learning~\cite{min2022rethinking}, we first utilize GPT-4 to generate a substantial dataset of triple samples, each comprising a source text, a target text, and an associated editing task label. Subsequently, leveraging the text embeddings extracted by the CLIP text encoder, we design a Transformer-based editing task classifier as shown in Fig.~\ref{fig:main_arch}. This classifier takes as input the concatenated embeddings of the source and target texts. After processing through multiple Transformer blocks, a classification head at the final layer outputs the predicted editing task category.}

\textcolor{black}{
With this task classifier that is trained on the dataset, the editing process becomes significantly more user-friendly. The user is required to input an image. We first employ BLIP~\cite{li2022blip} to generate a caption, providing a source text description. The user then inputs a desired target text. Our task classifier, receiving both the source and target texts as input, predicts the corresponding editing task type (one of the five categories). Based on this classification, our system dynamically selects the appropriate editing method for our Dual-Schedule Inversion, simplifying user interaction and enhancing the editing outcome by automatically adapting the strategy to match the identified task.
}
\vspace{-5pt}

\section{Experiments}
\vspace{-3pt}
We evaluate Dual-Schedule Inversion's performances of reconstruction in Sec.~\ref{rec} and editing in Sec.~\ref{sec:edit} on real images using the pre-trained popular Stable Diffusion model \cite{rombach2022high}. Due to the lack of public benchmarks for this evaluation, we build a testing set in this work.  The set has a total of 150 image-text pairs, among which 32 pairs are from all the examples used by three related works\footnote{timothybrooks.com/instruct-pix2pix}$^,$\footnote{ https://ljzycmd.github.io/projects/MasaCtrl}$^,$\footnote{https://github.com/csyxwei/ELITE}~\cite{cao2023masactrl,wei2023elite,brooks2022instructpix2pix}, and the rest of 118 pairs are from the Internet. It includes images of animals, humans, man-made objects, and scenes, some of which are of low quality or contain complex textures, increasing the possibility of distortion in reconstruction. All images are interpolated, cut, and/or scaled to the size of 512$\times$512. All the image-text pairs are listed in the supplementary material. For the implementation of our Dual-Schedule Inversion, the primary schedule is $[1, 21, 41, ..., 961, 981]$ and the auxiliary schedule is $[10, 30, 50, ..., 990]$.

 \begin{table*}[ht]
\centering
  \renewcommand\arraystretch{0.8}
\small
\setlength{\tabcolsep}{3.0mm}{
\begin{tabular}{c|c|c|cccc|cc}
\toprule
\toprule
\multicolumn{2}{c|}{Method}           & Structure          & \multicolumn{4}{c|}{Background Preservation} & \multicolumn{2}{c}{CLIP Similariy} \\ \midrule
Inversion          & Editing            & Distance$_{^{\times 10^3}}$ $\downarrow$ & PSNR $\uparrow$     & LPIPS$_{^{\times 10^3}}$ $\downarrow$  & MSE$_{^{\times 10^4}}$ $\downarrow$     & SSIM$_{^{\times 10^2}}$ $\uparrow$    & Whole  $\uparrow$          & Edited  $\uparrow$       \\ \midrule
AIDI& P2P                & 39.33           & 22.84  & 89.50  & 52.2  & 71.38 & 23.50        & 20.31          \\
NMG& P2P                & 14.15           & 26.02  & 42.32  & 24.9  & 76.84 & 25.17        & 22.90          \\
EDICT& P2P                &  15.20         &  \underline{26.30} &  \underline{39.49}  & \underline{23.4} & \textbf{78.50} & 24.62     & 21.87     \\
\textcolor{black}{ProxEdit}& P2P                &  \underline{13.70} & 24.35 & 69.58 & 36.9 & 73.40 & 23.80 & 21.70    \\
DirectInv& P2P                &  14.56 & 26.10 & 46.50 & 24.5 & 77.30 & \underline{26.70} & \textbf{23.84}    \\
DDIM& P2P                &  68.70 & 20.72 & 102.97 & 84.6 & 66.40 & 23.80 & 20.07   \\
\midrule
Ours& P2P                &  \textbf{13.09} & \textbf{26.35} & \textbf{39.50} & \textbf{23.3} & \underline{78.12} & \textbf{27.50} & \underline{23.00}    \\
\bottomrule
\end{tabular}}
\vspace{-6pt}
\caption{Comparison with SOTA inversion methods on editing.}
\label{tab_re1}
    \vspace{-0.5cm}
\end{table*}

\begin{table}[t]
    \centering
    \small
    \setlength{\tabcolsep}{1.7mm}{
    \begin{tabular}{l|c|c|c}
    \toprule
        Method & Tuning-free &  PSNR $\uparrow$ & SSIM $\uparrow$ \\ \midrule 
        DDIM Inversion & \CheckmarkBold & 16.348 & 0.509 \\
        Negative-Prompt Inversion & \CheckmarkBold & 21.352 & 0.624 \\
        Null-Text Inversion & \XSolidBrush & 26.106 & 0.738 \\ 
        Dual-Schedule Inversion & \CheckmarkBold & 25.977 & 0.738 \\ \midrule
        Upper Bound &  & 26.310 & 0.742 \\ \bottomrule
    \end{tabular}
    }    \vspace{-5pt}
    \caption{Reconstruction performance. The values are computed on the testing dataset with three guidance scales.}
    \label{tab-metric1}
    \vspace{-9pt}
\end{table}

\vspace{-2pt}
\subsection{Reconstruction Performance}\label{rec}
\vspace{-3pt}
In the reconstruction experiment, we compare Dual-Schedule Inversion with three recent inversion methods, DDIM Inversion~\cite{song2020denoising} (baseline), Null-Text Inversion~\cite{mokady2022null}, and Negative-Prompt Inversion~\cite{miyake2023negative}, under different guidance scales $w$ in Eq.~\ref{cfg}. To ensure an unbiased comparison, all experiments are conducted with fair sampling steps. Specifically, we maintain identical sampling times by using 50 steps for each of our two schedules and 100 steps for other methods. The quantitative comparison is presented in Table~\ref{tab-metric1}, and two qualitative examples are in Fig.~\ref{fig-rec} (more examples are provided in the supplementary material). We calculate the average PSNR and SSIM on this testing set across three different guidance scales 1.0, 4.0, and 7.5. 

In Fig.~\ref{fig-rec}, we clearly see that DDIM Inversion leads to obvious reconstruction failures. Most of its reconstructed images are very different from the source images, and its PSNR and SSIM values are far below those of the other methods (Table~\ref{tab-metric1}). Thus, editing real images using this baseline can lead to severe distortions, including object/scene identity distortions (see Sec.~\ref{sec:edit}). Our Dual-Schedule Inversion significantly outperforms Negative-Prompt Inversion quantitatively. Compared with Null-Text Inversion, our method has the same SSIM, with a slightly lower PSNR (25.977 vs. 26.106). However, Null-Text Inversion requires an extra fine-tuning stage.

In Table~\ref{tab-metric1}, we also show the upper bound for the reconstruction quality. Stable Diffusion model uses an encoder to encode an image to a latent space, in which sampling is performed, and uses a decoder to decode the final latent into the image space. So the upper bound is obtained without the sampling stage. We can see that the results of Dual-Schedule Inversion and Null-Text Inversion are close to the upper bound. Though our method is mathematically reversible, the small gap between it and the upper bound comes from the numerical errors during sampling.

To demonstrate the robustness and efficacy of our inversion method during editing, compared to other SOTA inversion techniques, we adopt the metrics from \cite{ju2023direct} to evaluate reconstruction ability during editing on 20 real images. Specifically, the structure distance and background preservation metrics in Table \ref{tab_re1} represent the reconstruction performance of the unedited parts. \textcolor{black}{Structure distance computes the distance between structure features extracted from DINO-ViT. For background preservation, \cite{ju2023direct} computes PSNR, LPIPS, MSE, and SSIM in the unedited area with manual-annotated masks. Following \cite{ju2023direct}, we also obtain these masks on the source image in the same way.} We uniformly equip the six inversion methods and ours with the same editing technique (P2P) to test their effectiveness. As shown in Table \ref{tab_re1}, our inversion method overall achieves the best performance. To explore the performance improvements of our inversion across different editing techniques, we conduct experiments on three editing methods. As shown in Table \ref{tab_re2}, the results demonstrate that our inversion method enhances all these editing techniques.
 \begin{table*}[ht]
\centering
\small
  \renewcommand\arraystretch{0.8}
\setlength{\tabcolsep}{3.0mm}{
\begin{tabular}{c|c|c|cccc|cc}
\toprule
\toprule
\multicolumn{2}{c|}{Method}           & Structure          & \multicolumn{4}{c|}{Background Preservation} & \multicolumn{2}{c}{CLIP Similariy} \\ \midrule
Inversion          & Editing            & Distance$_{^{\times 10^3}}$ $\downarrow$ & PSNR $\uparrow$     & LPIPS$_{^{\times 10^3}}$ $\downarrow$  & MSE$_{^{\times 10^4}}$ $\downarrow$     & SSIM$_{^{\times 10^2}}$ $\uparrow$    & Whole  $\uparrow$          & Edited  $\uparrow$       \\ \midrule
DDIM& P2P                &  68.70 & 20.72 & 102.97 & 84.6 & 66.40 & 23.80 & 20.07   \\
Ours& P2P                &  \textbf{13.09} & \textbf{26.35} & \textbf{39.50} & \textbf{23.3} & \textbf{78.12} & \textbf{27.50} & \textbf{23.00}    \\
\midrule
DDIM& MasaCtrl                &  70.80 & 21.26 & 97.54 & 74.1 & 69.27 & 21.92 & 20.38    \\
Ours& MasaCtrl                & \textbf{13.95}           & \textbf{25.96}  & \textbf{43.50}  & \textbf{25.4}  & \textbf{77.50} & \textbf{23.20}  & \textbf{21.50}          \\
\midrule
DDIM& SDEdit & 75.70 & 18.85  & 117.53 & 130.3 & 64.83 & 22.18 & 20.95 \\
Ours& SDEdit & \textbf{14.60} & \textbf{25.65}  & \textbf{47.51} & \textbf{27.2} & \textbf{76.50} & \textbf{24.05} & \textbf{21.78} \\
\bottomrule
\end{tabular}}
\vspace{-8pt}
\caption{Performance improvements of three editing methods using our inversion technique.}
\label{tab_re2}
    \vspace{-0.6cm}
\end{table*}
\begin{figure}[!t]
    \centering
\includegraphics[width=1.03\linewidth]{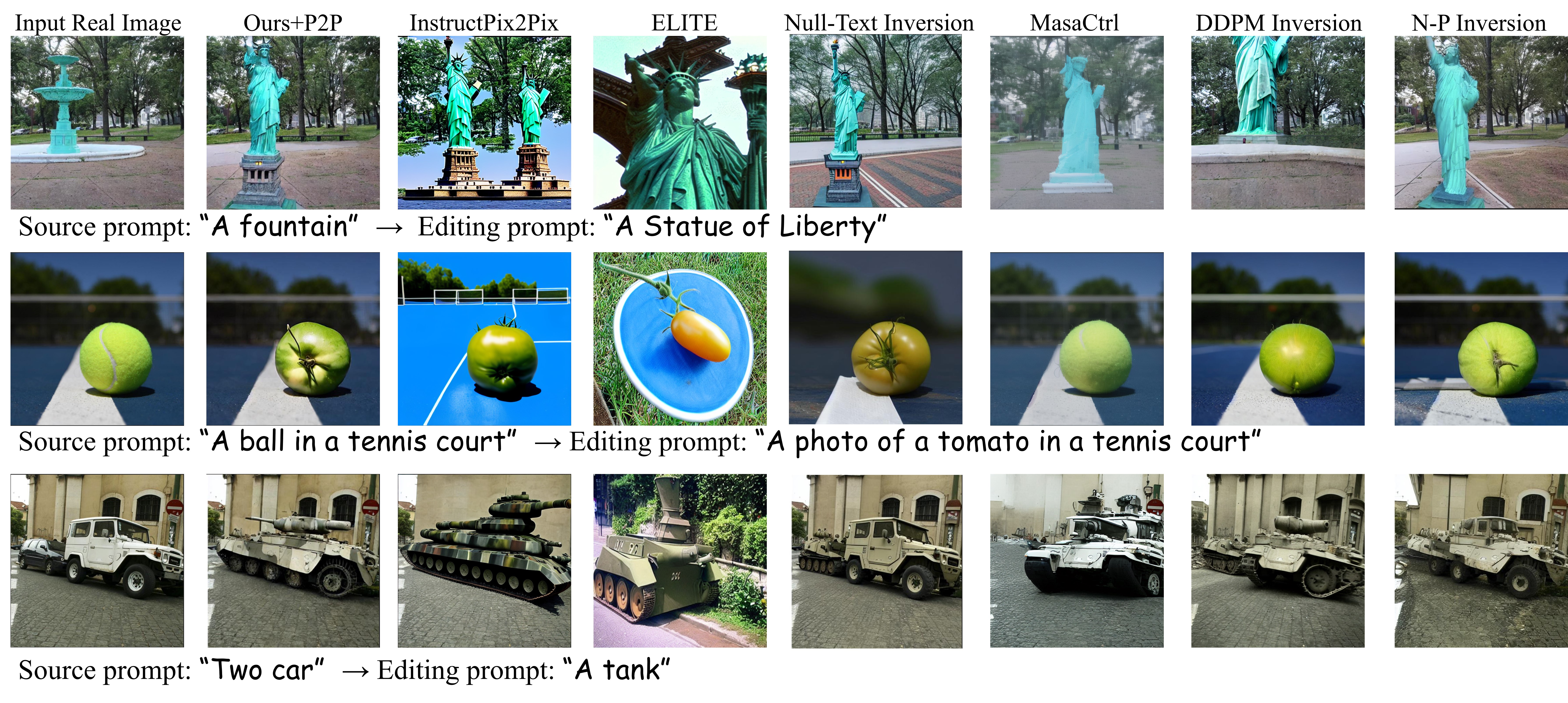}
    \vspace{-22pt}
    \caption{Comparison with SOTA editing methods on real images for \textit{Object Replacement}. N-P Inversion denotes Negative-Prompt Inversion.}
    \vspace{-12pt}
    \label{fig:results_rep}
\end{figure}
\begin{figure}[t]
    \centering
    \includegraphics[width=1.03\linewidth]{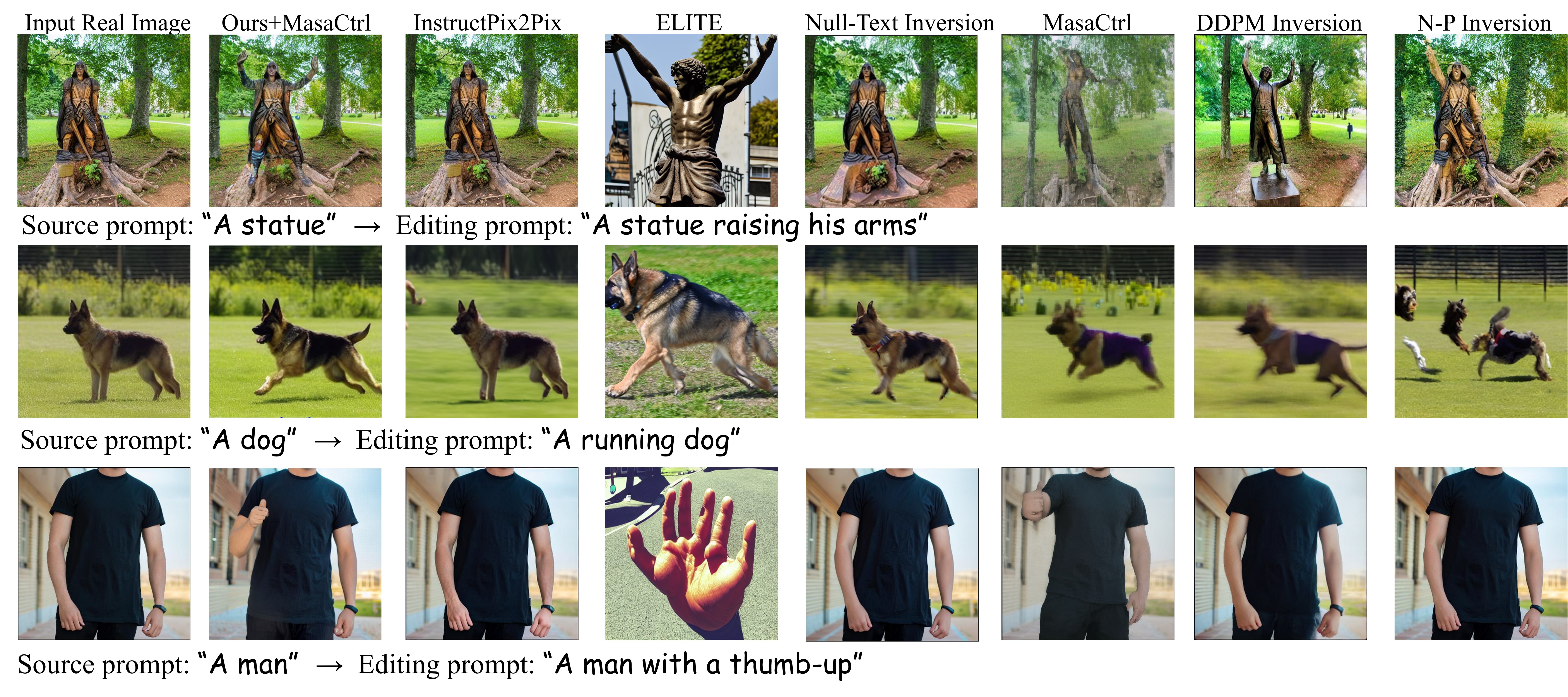}
    \vspace{-20pt}
    \caption{Comparison with SOTA editing methods on real images for \textit{Action Editing}.}
    \label{fig:results_act}
    \vspace{-4pt}
\end{figure}
\begin{figure}[t]
    \centering
    \includegraphics[width=1.03\linewidth]{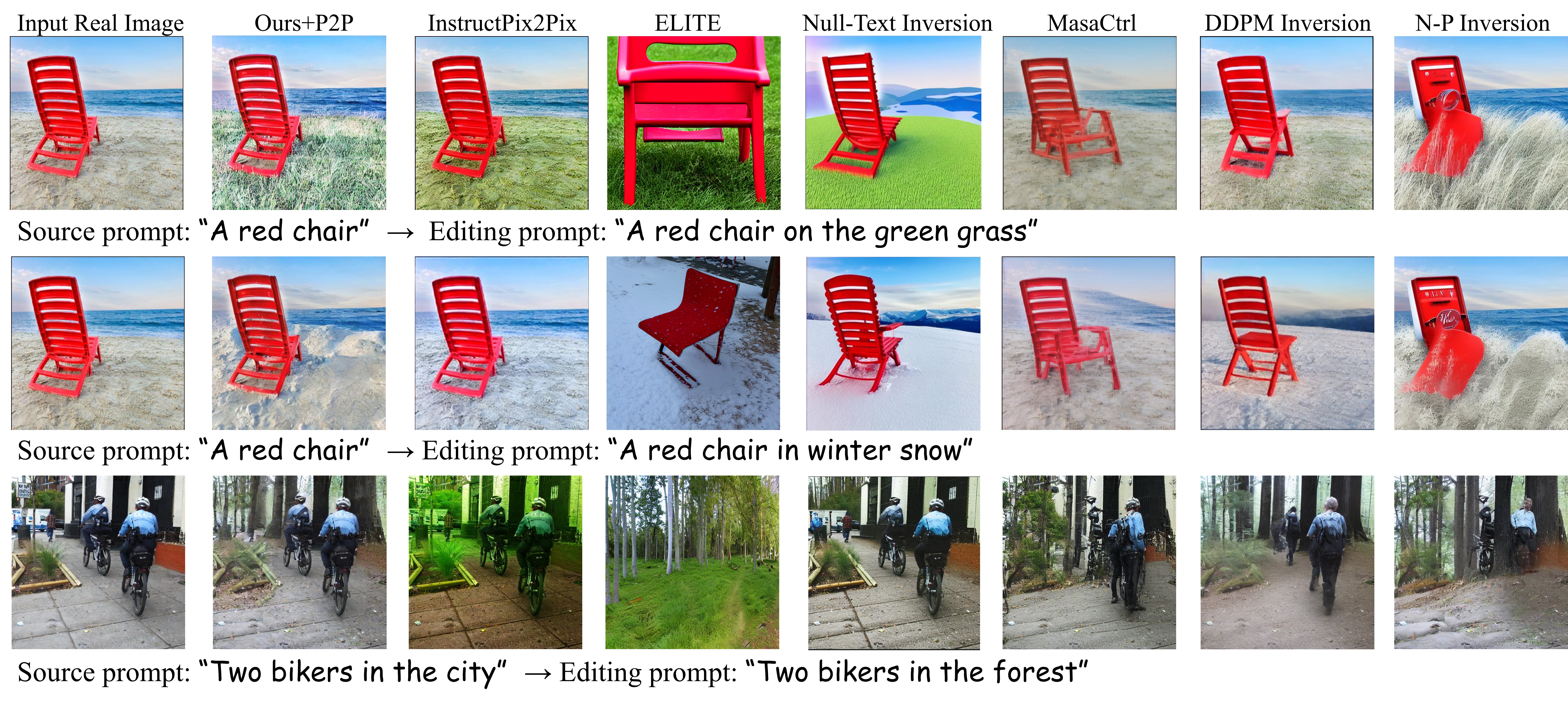}
    \vspace{-20pt}
    \caption{Comparison with SOTA editing methods on real images for \textit{Scene Editing}.}
    \label{fig:results_sce}
    \vspace{-15pt}
\end{figure}

 \begin{figure}[t]
    \centering
    \includegraphics[width=1.03\linewidth]{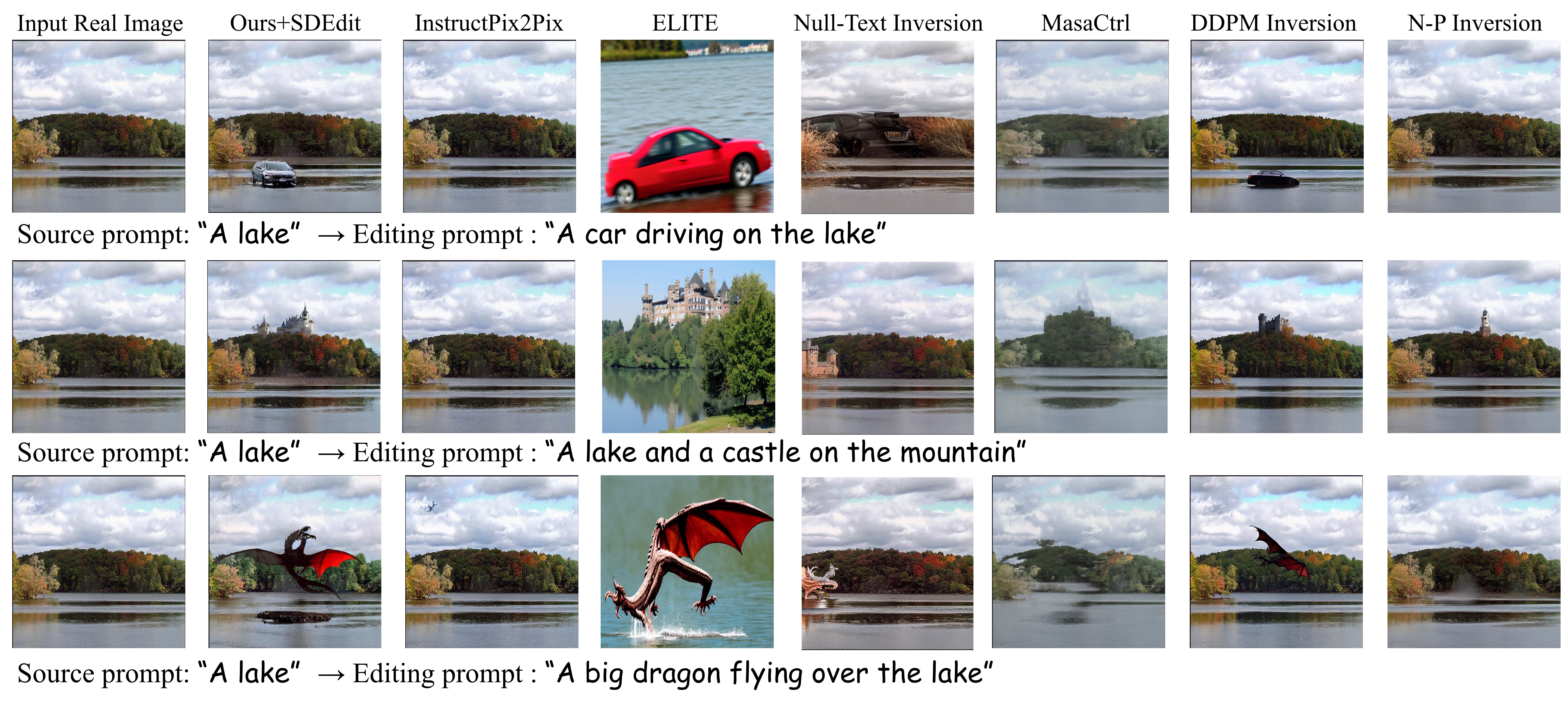}
    \vspace{-20pt}
    \caption{Comparison with SOTA editing methods on real images for \textit{Object Adding}.}
    \label{fig:results_add}
    \vspace{-4pt}
\end{figure}
\begin{figure}[t]
    \centering
\includegraphics[width=1.03\linewidth]{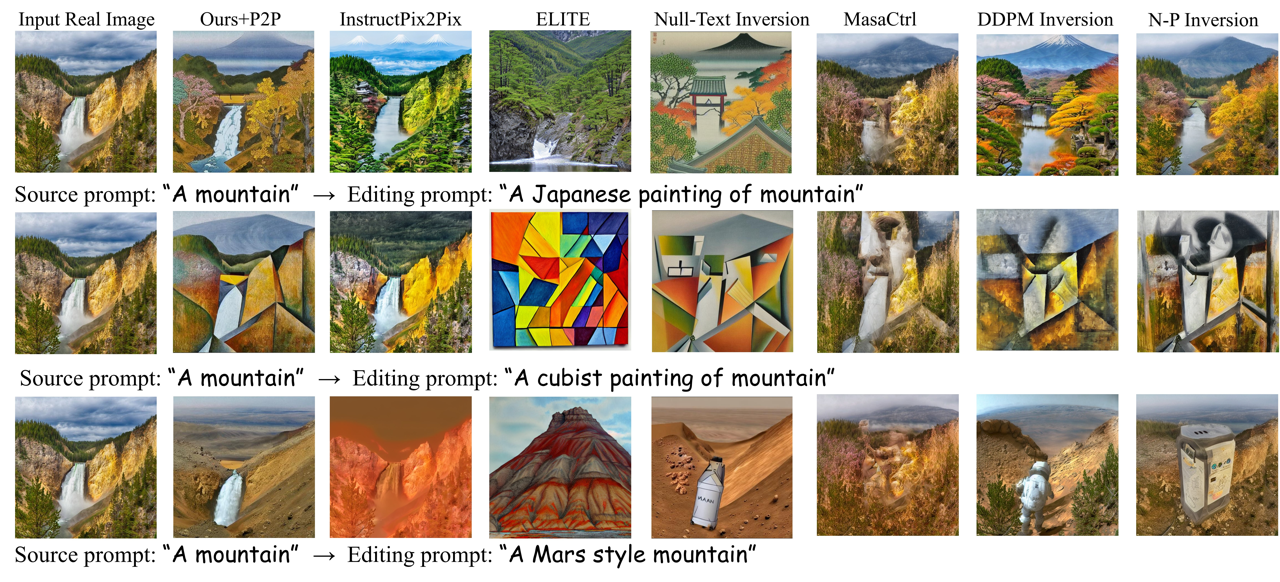}
    \vspace{-20pt}
    \captionof{figure}{
    Comparison with SOTA editing methods on real images for \textit{Style Editing}.
    }
    \vspace{-5pt}
    \label{fig:style}
\end{figure}
\begin{table}[t!]
    \centering
    \small
    \setlength{\tabcolsep}{0.55mm}{
    \resizebox{\linewidth}{!}{
    \begin{tabular}{l|ccccc}
    \toprule
      Method  & Null-Text Inversion & DreamBooth& ELITE  & N-P Inversion\\
      \hline
      User Score $\uparrow$ & 5.15&2.89& 1.95&1.77 \\ 
      CLIP Score $\uparrow$ & 0.2819& 0.2599 & 0.2509  &0.3031 & \\
      LPIPS $\downarrow$ & 0.035& 0.1663 & 0.1545  &0.0866 & \\
      \midrule
      Method   & InstructPix2Pix & MasaCtrl & DDPM Inversion & Ours\\ \hline
      User Score $\uparrow$& 4.99& 4.15&5.48 &\textbf{5.85} \\ 
      CLIP Score $\uparrow$  & 0.2401 & 0.2529 &0.2614 &\textbf{0.3041} \\
      LPIPS $\downarrow$& 0.097 & 0.0778&0.0715 &\textbf{0.018} \\
      \bottomrule
    \end{tabular}}}
        \vspace{-7pt}
    \caption{Quantitative results of different editing methods. The user scores are in the range of $[0,10]$. N-P Inversion denotes Negative-Prompt Inversion.}
    \vspace{-14pt}
    \label{tab-time}
\end{table}

\subsection{Editing Performance}\label{sec:edit}
\vspace{-3pt}
 
We compare with seven SOTA editing methods: InstructPix2Pix~\cite{brooks2022instructpix2pix}, MasaCtrl~\cite{cao2023masactrl}, ELITE \cite{wei2023elite}, DDPM Inversion~\cite{huberman2023edit}, DreamBooth~\cite{ruiz2022dreambooth}, Negative-Prompt Inversion~\cite{miyake2023negative}, and Null-Text Inversion~\cite{mokady2022null}. The visual comparisons for the five editing categories are shown in Figs.~\ref{fig:results_rep}--\ref{fig:style}, respectively. Most of the compared methods fail to solve the following two problems simultaneously: 1) the unedited part of an image cannot retain its original identity; 2) the edited results do not match the editing prompts. Specifically, ELITE, MasaCtrl, DDPM Inversion, and Negative-Prompt Inversion often result in identity distortions. InstructPix2Pix and Null-Text Inversion tend to retain the structure of the object, so they are not able to perform action editing. MasaCtrl~\cite{cao2023masactrl} is designed for action editing and it is not able to perform object or scene replacement. Our method is able to solve these two problems simultaneously in various types of editing because of its perfect reconstruction capability.

\begin{table*}[t]
\centering
\small
\vspace{3pt}\renewcommand\arraystretch{0.8}
\setlength{\tabcolsep}{2.8mm}{
\begin{tabular}{c|c|c|cccc|cc}
\toprule
\toprule
\multicolumn{2}{c|}{Method}           & Structure          & \multicolumn{4}{c|}{Background Preservation} & \multicolumn{2}{c}{CLIP Similariy} \\ \midrule
Inversion          & Editing            & Distance$_{^{\times 10^3}}$ $\downarrow$ & PSNR $\uparrow$     & LPIPS$_{^{\times 10^3}}$ $\downarrow$  & MSE$_{^{\times 10^4}}$ $\downarrow$     & SSIM$_{^{\times 10^2}}$ $\uparrow$    & Whole  $\uparrow$          & Edited  $\uparrow$       \\ \midrule
Ours& P2P                &  13.09 & 26.35 & 39.50 & 23.3 & 78.12 & 27.50 & 23.00    \\
Ours& MasaCtrl                & 13.95           & 25.96  & 43.50  & 25.4  & 77.50 & 23.20        & 21.50          \\
Ours& SDEdit & 14.60 & 25.65  & 47.51 & 27.2 & 76.50 & 24.05 & 21.78 \\
\midrule
Ours             & Classifier                & \textbf{12.98}          & \textbf{26.42}  & \textbf{38.17}  & \textbf{22.80}  & \textbf{78.93} & \textbf{30.41}         & \textbf{25.43}          \\
\bottomrule
\end{tabular}}
\vspace{-6pt}
\caption{Ablation study for our classifier.}
\label{tab_re3}
    \vspace{-0.3cm}
\end{table*}

 \begin{figure}[t]
    \centering
    \includegraphics[width=1.03\linewidth]{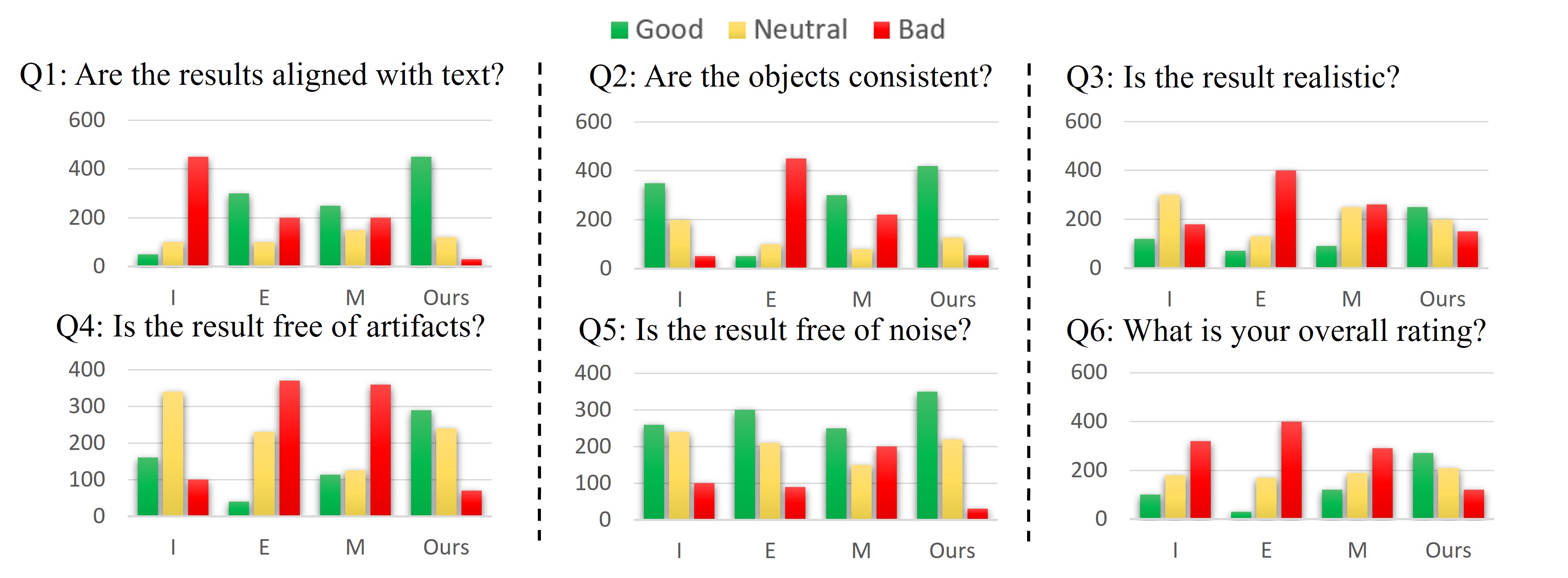}
    \vspace{-20pt}
    \caption{Comparison with three most recent SOTA methods for 6 specific questions with 3 possible answers (good, neutral, bad), where I, E, and M denote InstructPix2Pix, ELITE, and MasaCtrl, respectively.}
    \label{fig:usertab}
    \vspace{-15pt}
\end{figure}

For quantitative evaluation, since there is no image editing benchmark, similar to~\cite{wei2023elite}, we conduct a user study with 30 participants to evaluate 180 edited results by the 8 methods on 20 real images. The participants are asked to score the overall editing quality of these results from 1 (worst) to 10 (best), based on the following two criterions: 1) consistency of the results with the original image, and 2) consistency of the results with the editing prompt. For each image, 8 edited results by the methods are shown to each participant. The average user scores, CLIP Score~\cite{hessel2021clipscore}, and LPIPS~\cite{zhang2018unreasonable} are reported in Table~\ref{tab-time}, where our method obtains the best scores (see DiffEdit~\cite{couairon2022diffedit} about how to calculate CLIP Score and LPIPS). In order to analyze the detailed aspects influencing these scores, we specifically design 6 Q\&As and compare with 3 most recent methods, InstructPix2Pix~\cite{brooks2022instructpix2pix}, MasaCtrl~\cite{cao2023masactrl}, and ELITE~\cite{wei2023elite}. Fig.~\ref{fig:usertab} reports the rating distributions of the 4 methods, among which ours receives more “good” and fewer “bad” answers.

\subsection{Ablation Study}\label{sec:ablation_study}

\noindent \textbf{Auxiliary Schedules.}\quad In Fig.~\ref{fig:main_arch} and the reconstruction experiments, the primary schedule is $[1, 21, 41, ..., 961, 981]$ and the auxiliary schedule is $[10, 30, 50, ..., 950, 970]$, where $\tau=t-s+\frac{s}{2}$ (see Eq.~\ref{what sample sim} and Eq.~\ref{what invert sim}). In fact, $\tau$ can be set to other values such as $\tau=t-s+\frac{1}{4}s$ and $\tau=t-s+\frac{3}{4}s$. From the second column in Table~\ref{tab:abs}, we can see that the reconstruction performances are very similar in these three settings of $\tau$. However, we find that when $\tau=t-s+\frac{1}{4}s$ or $\tau=t-s+\frac{3}{4}s$, the editing performances are not as good as when $\tau=t-s+\frac{s}{2}$ (see the supplementary material). We suspect that when $\tau=t-s+\frac{s}{2}$ (i.e, the midpoint between $t-s$ and $t$), the auxiliary latent $\tilde{z}^a_\tau$ can better balance the sampled primary latents $\tilde{z}^p_{t-s}$ and $\tilde{z}^a_t$.
\begin{table}[t]
\renewcommand{\arraystretch}{1.0}
\centering
\footnotesize
\setlength{\tabcolsep}{0.4mm}{
\begin{tabular}{c|c c c| c c}
\toprule
$T^{\prime}$ & 50 & 50 & 50 &  50 & 20\\
\hline
$\tau$ & $t-s+\frac{1}{4}s$\quad\quad & $t-s+\frac{3}{4}s$\quad\quad & $t-s+\frac{s}{2}$ &$t-s+\frac{s}{2}$ \quad& $t-s+\frac{s}{2}$\\
\midrule
PSNR  $\uparrow$ &25.961&25.963 & 25.977 & 25.977 & 25.945\\
SSIM $\uparrow$ & 0.737 &0.737 & 0.738 & 0.738 & 0.732\\
\bottomrule
\end{tabular}}
\vspace{-6pt}
\caption{Ablation Study for $\tau$ and $T^{\prime}$.}
\label{tab:abs}
\vspace{-6pt}
\end{table}

\noindent \textbf{Schedule Lengths.}\quad In Fig.~\ref{fig:main_arch} and all the editing examples, the numbers of the sampling steps are all 50 ($T^\prime=50$, defined in Sec.~\ref{sec:fail}). Our Dual-Schedule Inversion can be applied to fewer or more steps. In the third column of Table~\ref{tab:abs}, we compare its reconstruction performances when $T^\prime=50$ and $T^\prime=20$, and see that the performance drops only slightly when $T^\prime=20$, showing that it is robust to different sampling steps.

\noindent \textbf{Task Classifier.}\quad \textcolor{black}{To further demonstrate the efficacy of the task classifier within our pipeline, we integrate our Dual-Schedule Inversion framework with SDEdit, P2P, and MasaCtrl, both with and without the implementation of the task classifier. This evaluation is carried out on our testing set comprising 150 images, distributed across the five editing tasks. The results are detailed in Table~\ref{tab_re3}, 
which reveals a notable drop in editing performance when the task classifier is not utilized due to the fact that each algorithm is only good at specific tasks. By adaptively selecting the most appropriate algorithm for each specific task, our inversion approach facilitates a more user-friendly application.}

\section{Conclusion and Limitation}

In this paper, we propose a tuning-free Dual-Schedule Inversion that enables faithful reconstruction and editing on real images. With the novel design of the primary and auxiliary schedules for inversion and sampling, our method guarantees reversibility mathematically. Combining it with other editing techniques, it achieves SOTA editing performance on object replacement, action editing, scene editing, new object creation, and style editing. 

Since it needs to be integrated with other editing techniques for image editing, its editing ability is restricted by these techniques. When a more powerful editing method appears, its editing competence will also be boosted.

{\small
\bibliographystyle{ieee_fullname}
\bibliography{main}
}

\clearpage
\setcounter{page}{1}
\pagenumbering{arabic}
\appendix
\twocolumn[{%
\renewcommand\twocolumn[1][]{#1}%
\maketitle
\vspace{-0.9cm}
\begin{center}
    \renewcommand*{\thefigure}{S\arabic{figure}}
    \centering
    \includegraphics[width=1\linewidth]{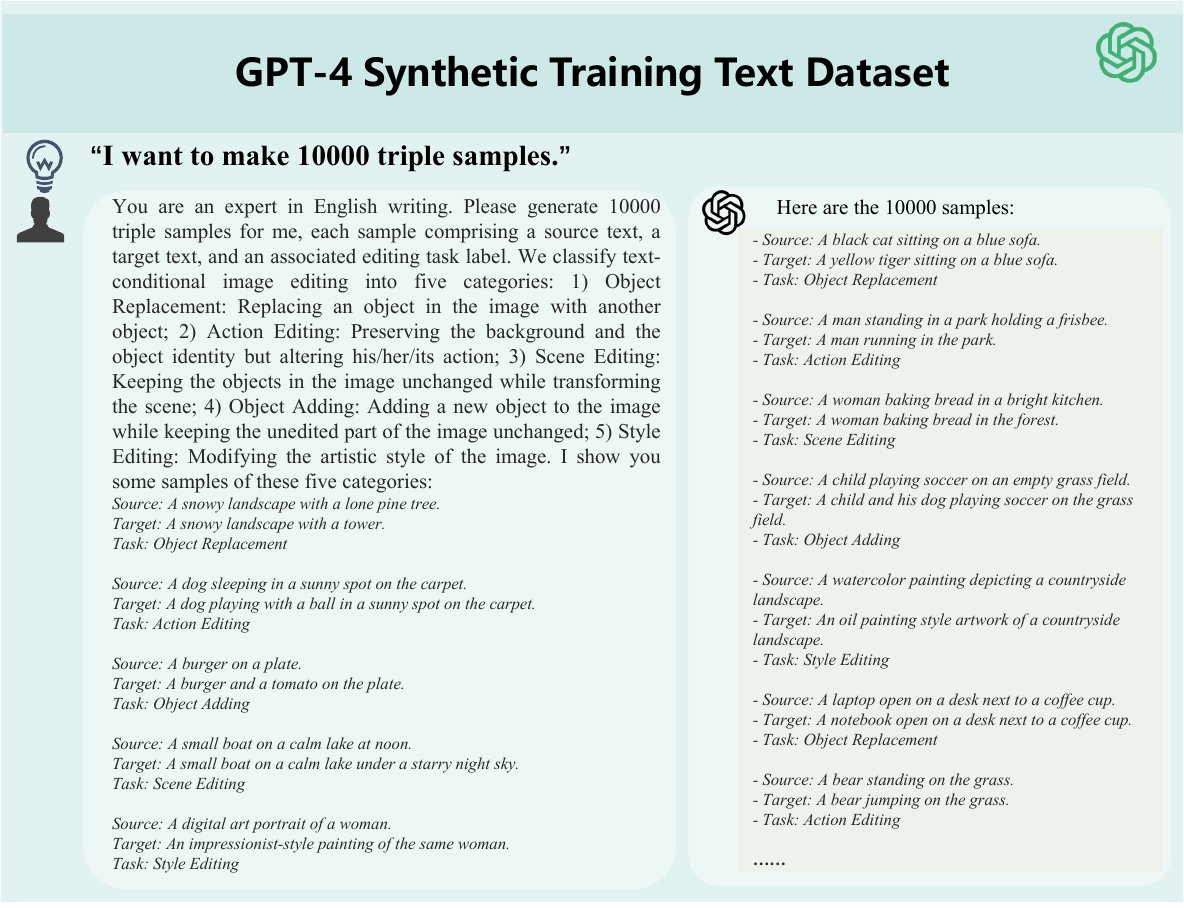}
    \vspace{-19pt}
    \captionof{figure}{
    Using GPT4 to generate our triple dataset.
    }
    \vspace{-0.1cm}
    \label{fig:gpt}
\end{center}%
}]

\begin{figure*}[t]
    \centering
    \renewcommand*{\thefigure}{S\arabic{figure}}
    \includegraphics[width=\linewidth]{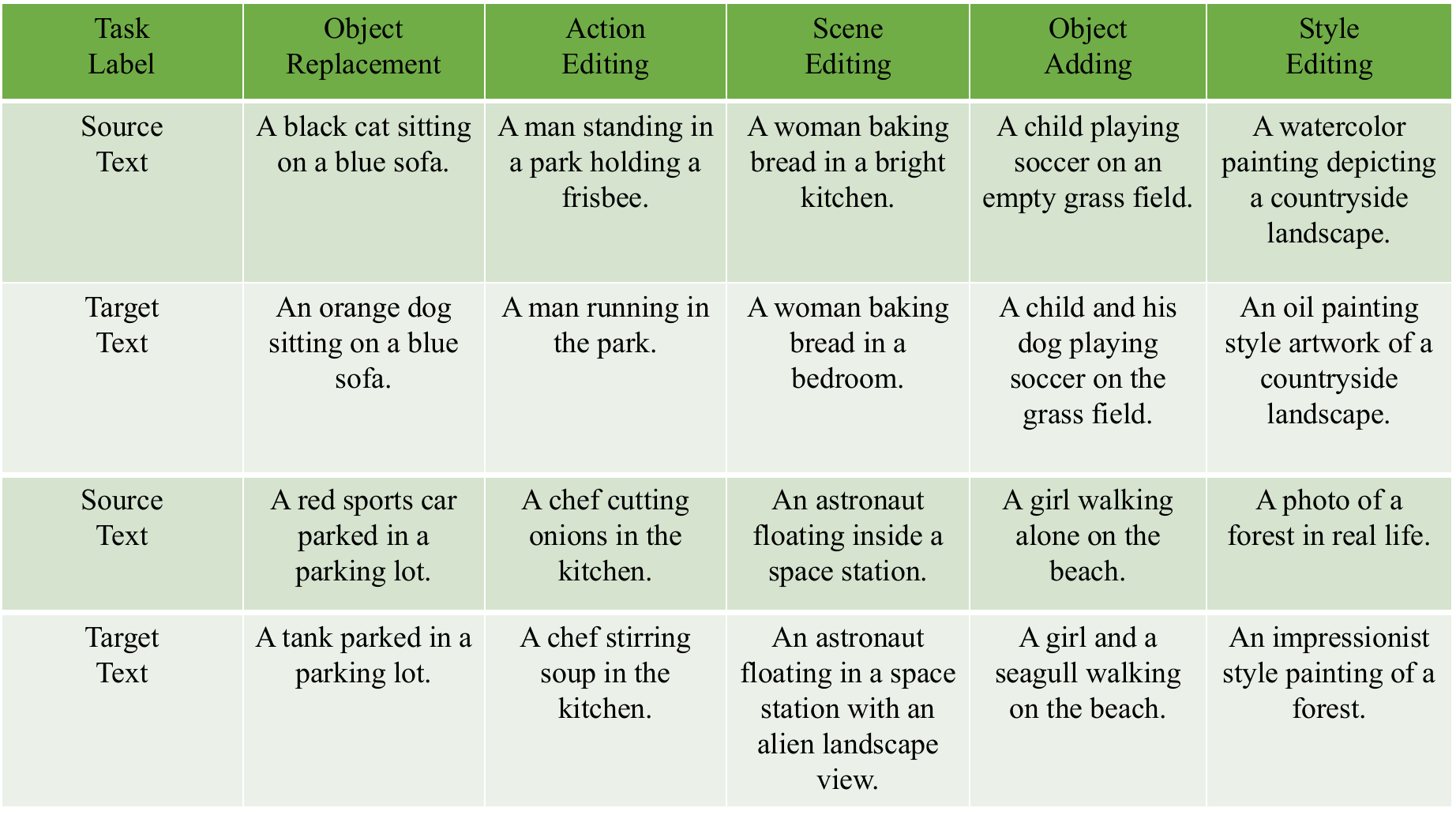}
    \vspace{-19pt}
    \caption{Examples of our triple dataset.}
    \label{fig:text}
    \vspace{1pt}
\end{figure*}
In this supplementary material, we (1) show how to build the training dataset for the editing task classifier and the structure of the classifier, (2) give more ablation study, (3) show all the test images with prompts, (4) prove $\bar z_1^{p}$ and $\bar z_{10}^{a}$ obtained by DDIM does not affect our reversibility, \textcolor{black}{(5) provide more reconstruction examples across different CFG scales, and (6) summarize recent inversion and editing methods in Table~\ref{tab-type}.}

\section{Editing Tasks Classification}
\label{sec:dataset}
In the main paper, we classify text-conditional image editing into five categories: 1) \textit{Object Replacement}: Replacing an object in the image with another object; 2) \textit{Action Editing}: Preserving the background and the object identity but altering his/her/its action; 3) \textit{Scene Editing}: Keeping the objects in the image unchanged while transforming the scene; 4) \textit{Object Adding}: Adding a new object to the image while keeping the unedited part of the image unchanged; 5) \textit{Style Editing}: Modifying the artistic style of the image. We first utilize GPT-4 to generate a substantial training dataset of triple samples as shown in Fig.~\ref{fig:gpt}, each comprising a source text, a target text, and an associated editing task label. We give some examples of our dataset in Fig.~\ref{fig:text}.

\begin{table}[t]
	\small
	\renewcommand\arraystretch{1.1}
	\renewcommand*{\thetable}{S\arabic{table}}
        \vspace{-1pt}
	\begin{center}
		
		\setlength{\tabcolsep}{0.4mm}{
			\begin{tabular}{c|cc}
				\toprule
				Network&Setting&Value\\
				\hline				
				\multirow{7}*{Classifier}&Input Channel&768\\
				~&Token Number & 154 (77$\times$ 2)\\
				~&Transformer block channels&512\\
				~&Attention head number&$8$\\
				~&Residual Connection&True\\
				~&Block numbers&6\\
                ~&Final linear layer input channels&512\\
                ~&Final linear layer output channels&5\\
				\bottomrule 
			\end{tabular}
		}
		
            \vspace{-4pt}
            \caption{Model configurations and parameter choices.
		}\label{tab:class}
		\vspace{-19pt}
	\end{center}
\end{table}

\begin{table}[ht]
    \centering
    \small
    \renewcommand*{\thetable}{S\arabic{table}}
    \setlength{\tabcolsep}{1.2mm}{
    \begin{tabular}{l|l|l}
    \toprule
    Type & Learning Strategy & Method \\ \midrule
    \multirow{10}{*}{Inversion} & \multirow{3}{*}{Testing-Time Finetuning} & Null-Text Inversion~\cite{mokady2022null} \\ 
     ~&~& Textual Inversion~\cite{gal2022image} \\
     ~&~& AIDI\cite{pan2023effective} \\
     \cmidrule{2-3}
     ~&\multirow{7}{*}{Training \& Tuning Free} & DDIM Inversion~\cite{song2020denoising} \\
     ~&~& DDPM Inversion~\cite{huberman2023edit} \\
     ~&~& N-P Inversion~\cite{miyake2023negative} \\ 
     ~&~& ProxEdit~\cite{han2023improving} \\ 
     ~&~& EDICT\cite{wallace2023edict} \\ 
     ~&~& NMG\cite{cho24noise} \\
     ~&~& DirectInv\cite{ju2023direct} \\ 
     ~&~& \textbf{Dual-Schedule Inversion} \\ \midrule
    \multirow{8}{*}{Editing} & \multirow{3}{*}{Training-Based} & ELITE~\cite{wei2023elite} \\ 
     ~&~& FastComposer~\cite{xiao2023fastcomposer} \\
     ~&~& InstructPix2Pix~\cite{brooks2022instructpix2pix} \\ 
     \cmidrule{2-3}
     ~&\multirow{2}{*}{Testing-Time Finetuning} & DreamBooth~\cite{ruiz2022dreambooth} \\ 
     ~&~& Custom Diffusion~\cite{kumari2022customdiffusion} \\ 
     \cmidrule{2-3}
     ~&\multirow{3}{*}{Training \& Tuning Free} & SDEdit~\cite{meng2021sdedit} \\ 
     ~&~& P2P~\cite{hertz2022prompt} \\ 
     ~&~& MasaCtrl~\cite{cao2023masactrl} \\ \bottomrule
    \end{tabular}}
    \vspace{-5pt}
    \caption{Some recent inversion and editing methods. N-P Inversion denotes Negative-Prompt Inversion~\cite{miyake2023negative}.}
    \label{tab-type}
    \vspace{-10pt}
\end{table}
Subsequently, leveraging the text embeddings extracted by the CLIP text encoder, we design a Transformer-based editing task classifier as shown in Table~\ref{tab:class}. This classifier takes as input the concatenated embeddings of the source and target texts. After processing through multiple Transformer blocks, a classification head at the final layer outputs the predicted editing task category.
\begin{figure*}[t]

    \renewcommand*{\thefigure}{S\arabic{figure}}
    \centering
    \includegraphics[width=1\linewidth]{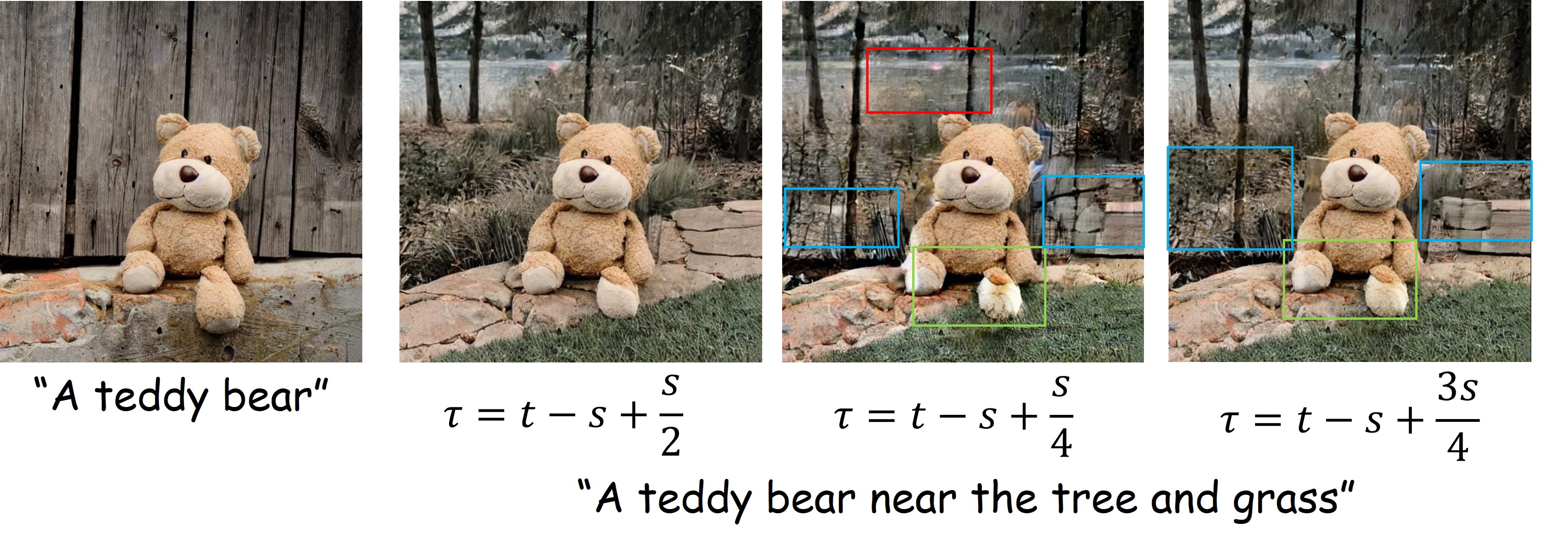}
    \vspace{-25pt}
    \caption{
    Editing results with $\tau=t-s+\frac{s}{2}$, $\tau=t-s+\frac{1}{4}s$, or $\tau=t-s+\frac{3}{4}s$. Slight blurring (red boxes), artifacts (green boxes), and strange textures (blue boxes) appear in the last two edited images.
    }
    \label{fig:abla}
\end{figure*}

\section{Results of Ablation Study on Auxiliary Schedules}
\label{sec:abla}
As mentioned in Sec.~5.3, the reconstruction performances are very similar in three settings of $\tau$. However, when $\tau=t-s+\frac{1}{4}s$ or $\tau=t-s+\frac{3}{4}s$, the editing performances are not as good as when $\tau=t-s+\frac{s}{2}$. For instance, we give a scene editing example in Fig.~\ref{fig:abla}. It can be seen that slight blurring, artifacts, and strange textures may be present in certain regions of the edited image when $\tau=t-s+\frac{1}{4}s$ or $\tau=t-s+\frac{3}{4}s$.

\section{Testing Dataset}
\label{sec:data}
Due to the lack of public benchmarks for the evaluation, we build a testing set in this work.  The set has a total of 150 image-text pairs, among which 32 pairs are from all the examples used by three related works\footnote{ timothybrooks.com/instruct-pix2pix}$^,$\footnote{ https://ljzycmd.github.io/projects/MasaCtrl}$^,$\footnote{https://github.com/csyxwei/ELITE}~\cite{cao2023masactrl,wei2023elite,brooks2022instructpix2pix}, and the rest of 118 pairs are from the Internet. All images are interpolated, cut, and/or scaled to the size of 512$\times$512. We use all the images for the reconstruction experiment and the first 20 images for the editing experiment. The complete set of the image-text pairs is presented in Figs.~\ref{fig:data1}, ~\ref{fig:data2}, ~\ref{fig:data3}, and \ref{fig:data4}.

\section{Initial Latent Setup and Reversibility Proof}
\label{sec:re}

In Section 4.2 of the main paper, we state that $\bar{z}_1^{p}$ and $\bar{z}_{10}^{a}$ obtained by the original forward process of DDIM do not affect our method's reversibility. To prove it, we need to ensure that these initial values conform to the conditions required for reversibility.

Mathematically, given $\bar{z}_1^{p}$ and $\bar{z}_{10}^{a}$ obtained by DDIM, we need to prove:
\begin{equation}
    \tilde{z}_{1}^{p} = \bar{z}_{1}^{p},
\end{equation}
which represents the reversibility of our inversion and sampling since $\bar{z}_{1}^{p}$ and $\tilde{z}_{1}^{p}$ denote the beginning of the inversion stage and the endpoint of the sampling stage, respectively.

\begin{proposition}
Let $\bar{z}_1^{p}$ and $\bar{z}_{10}^{a}$ be obtained by the original forward process of DDIM. Then $\tilde{z}_{1}^{p}=\bar{z}_{1}^{p}$ where $\tilde{z}_{1}^{p}$ is obtained by the Dual-Schedule Inversion method presented in Section 4.3 of the main paper.
\end{proposition}

\begin{proof}
Given $\bar{z}_1^{p}$ and $\bar{z}_{10}^{a}$, we can obtain $\bar{z}_{21}^{p}$ using Eq.~\ref{invert odd}:
\begin{equation}\label{supp:invert2}
\bar{z}_{21}^{p} = a_{(1\rightarrow 21)}\bar{z}_1^{p} + b_{(1\rightarrow 21)} \epsilon_\theta(\bar{z}_{10}^{a},10).
\end{equation}
Similarly, given $\bar{z}_{21}^{p}$ and $\bar{z}_{10}^{a}$, we can obtain $\bar{z}_{30}^{a}$ using Eq.~\ref{invert even}:
\begin{equation}\label{supp:invert}
\bar{z}_{30}^{a} = a_{(10\rightarrow 30)}\bar{z}_{10}^{a} + b_{(10\rightarrow 30)} \epsilon_\theta(\bar{z}_{21}^{p},21).
\end{equation}
By iterating this inversion process, we ultimately obtain:
\begin{equation}
\bar{z}_{981}^{p} = a_{(961\rightarrow 981)}\bar{z}_{961}^{p} + b_{(961\rightarrow 981)} \epsilon_\theta(\bar{z}_{970}^{a},970).
\end{equation}

From the Reversibility Requirement part in Section 4.2, the sampling process from $t=981$ is reversible in the inversion process. Thus, at time steps $t=30$ and $\tau=21$ during the sampling stage, we have $\tilde{z}_{30}^{a}=\bar{z}_{30}^{a}$ and $\tilde{z}_{21}^{p}=\bar{z}_{21}^{p}$. Given $\tilde{z}_{21}^{p}$ and $\tilde{z}_{30}^{a}$, we can compute:
\begin{equation}\label{supp:sample}
\tilde{z}_{10}^{a} = a_{(30\rightarrow 10)}\tilde{z}_{30}^{a} + b_{(30\rightarrow 10)} \epsilon_\theta(\tilde{z}_{21}^{p},21).
\end{equation}
Comparing Eq. \ref{supp:sample} with Eq. \ref{supp:invert}, we have $\tilde{z}_{10}^{a}= \bar{z}_{10}^{a}$ since these two equations are actually the same. 

Finally, given $\tilde{z}_{10}^{a}$ and $\tilde{z}_{21}^{p}$, we obtain:
\begin{equation}\label{supp:sample2}
\tilde{z}_{1}^{p} = a_{(21\rightarrow 1)}\tilde{z}_{21}^{p} + b_{(21\rightarrow 1)} \epsilon_\theta(\tilde{z}_{10}^{a},10).
\end{equation}
Again, comparing Eq. \ref{supp:sample2} with Eq. \ref{supp:invert2}, we have $\tilde{z}_{1}^{p}=\bar{z}_{1}^{p}$, which completes the proof.
\end{proof}

\section{More Reconstruction Examples}
\label{sec:rec} 
Here we provide 10 reconstruction examples across different guidance scales 1, 4, and 7.5 in Fig.~\ref{fig:rec1},~\ref{fig:rec2}, and ~\ref{fig:rec3}.

 \begin{figure*}[t]
    \centering
    \renewcommand*{\thefigure}{S\arabic{figure}}
    \includegraphics[width=\linewidth]{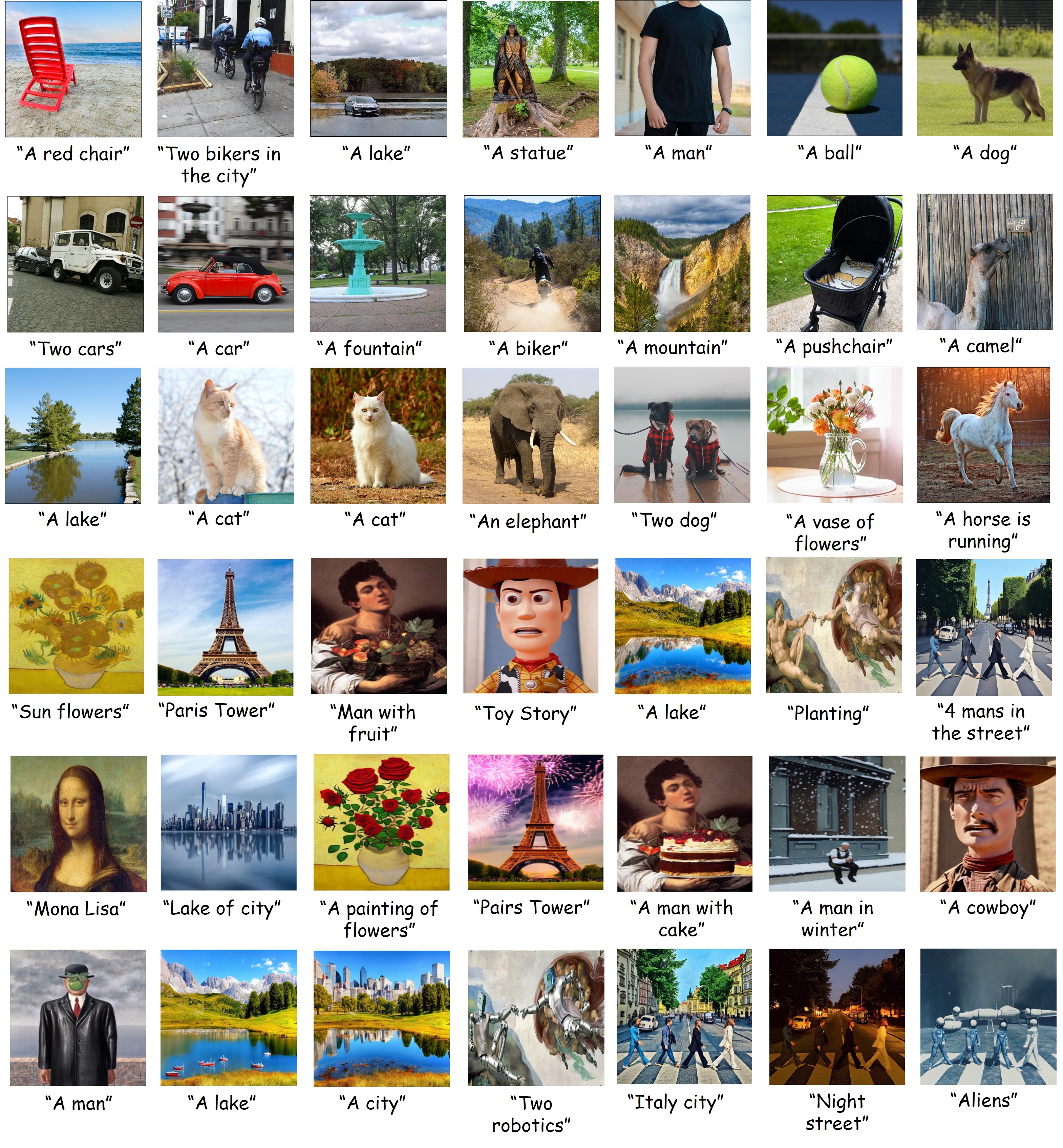}
    \vspace{-24pt}
    \caption{Images of the testing dataset (\textit{Part 1}).}
    \label{fig:data1}
    \vspace{-14pt}
\end{figure*}

 \begin{figure*}[t]
    \centering
    \renewcommand*{\thefigure}{S\arabic{figure}}
    \includegraphics[width=\linewidth]{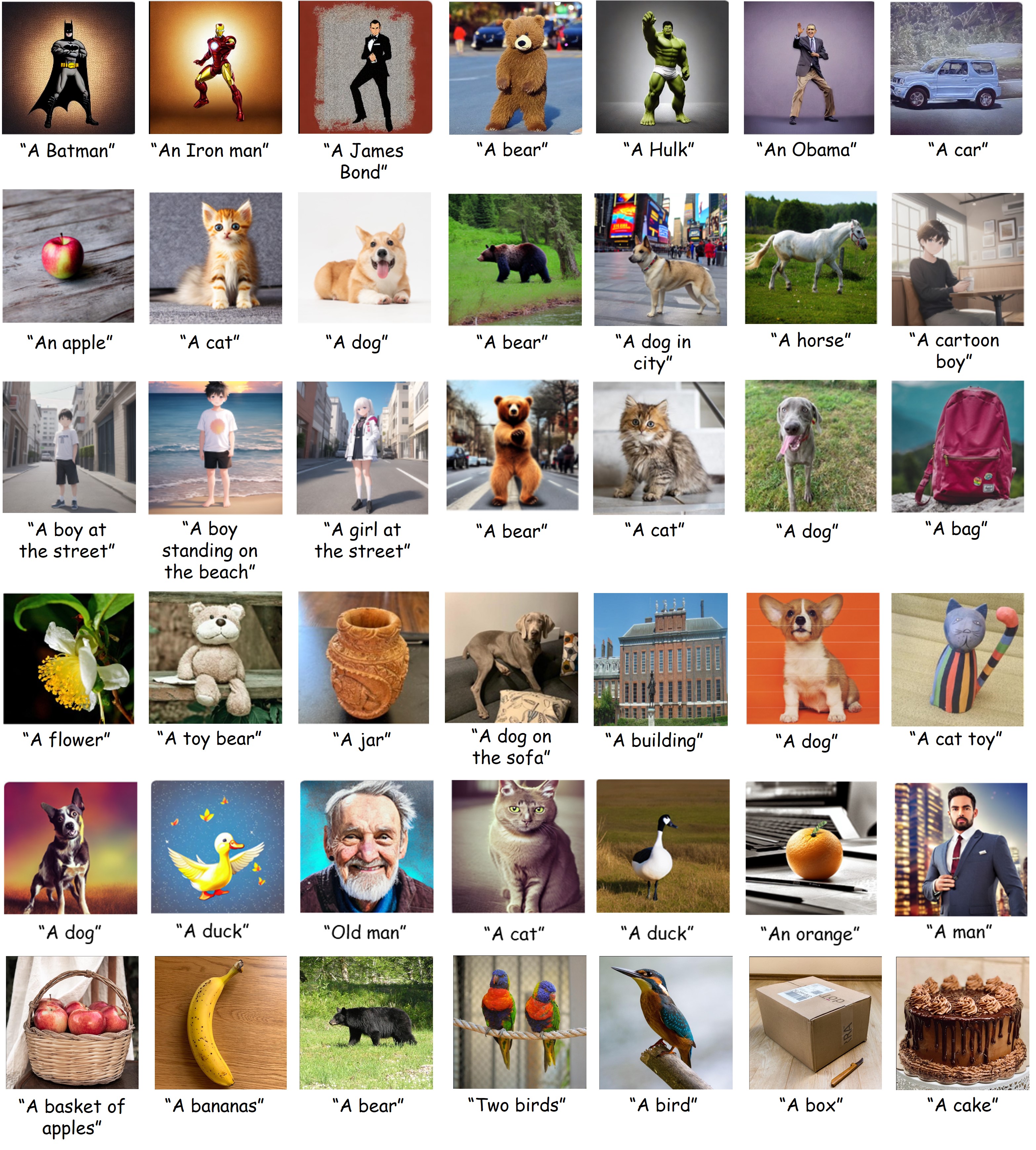}
    \vspace{-24pt}
    \caption{Images of the testing dataset (\textit{Part 2}).}
    \label{fig:data2}
    \vspace{-14pt}
\end{figure*}

 \begin{figure*}[t]
    \centering
    \renewcommand*{\thefigure}{S\arabic{figure}}
    \includegraphics[width=\linewidth]{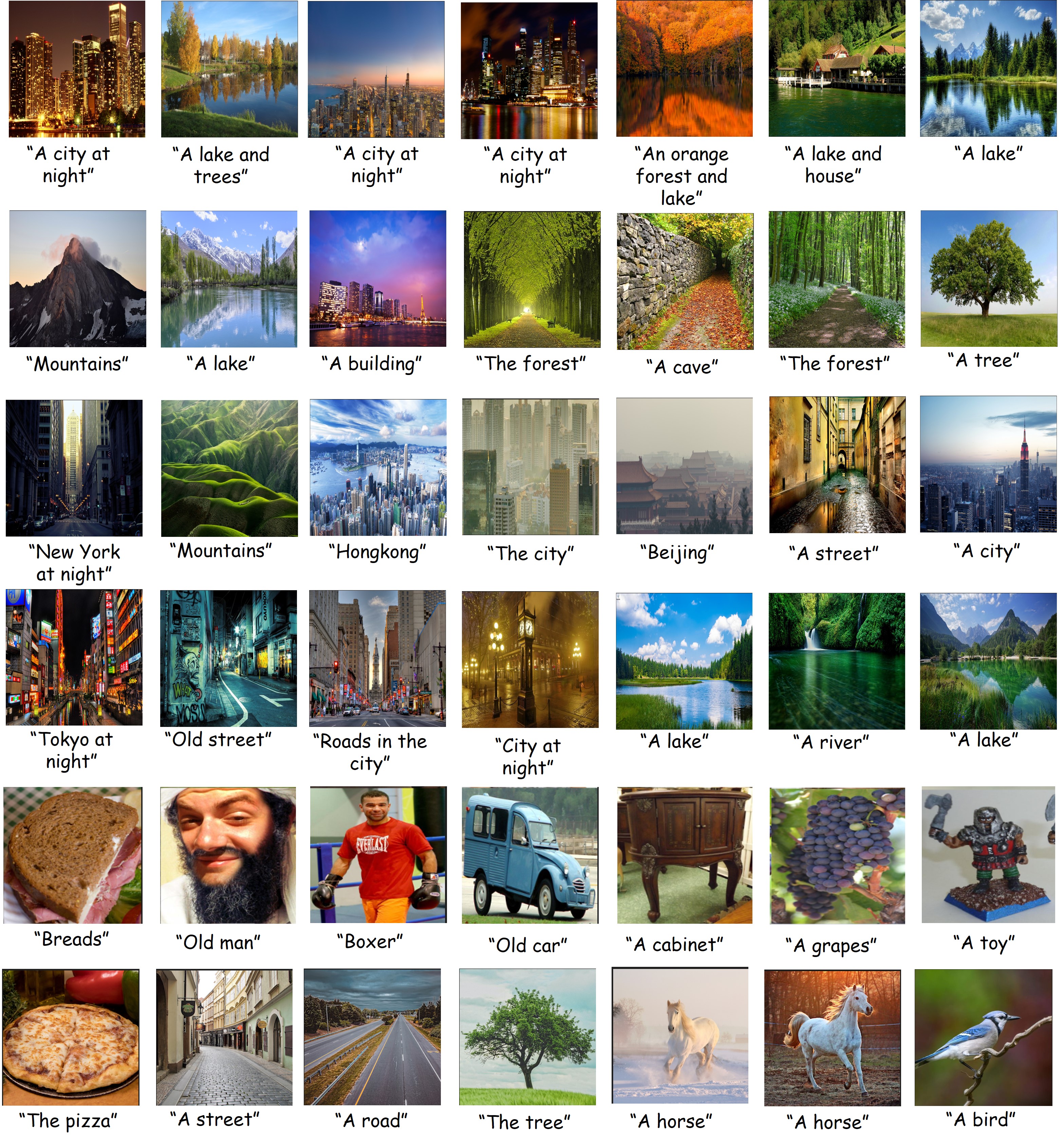}
    \vspace{-24pt}
    \caption{Images of the testing dataset (\textit{Part 3}).}
    \label{fig:data3}
    \vspace{-14pt}
\end{figure*}

 \begin{figure*}[t]
    \centering
    \renewcommand*{\thefigure}{S\arabic{figure}}
    \includegraphics[width=\linewidth]{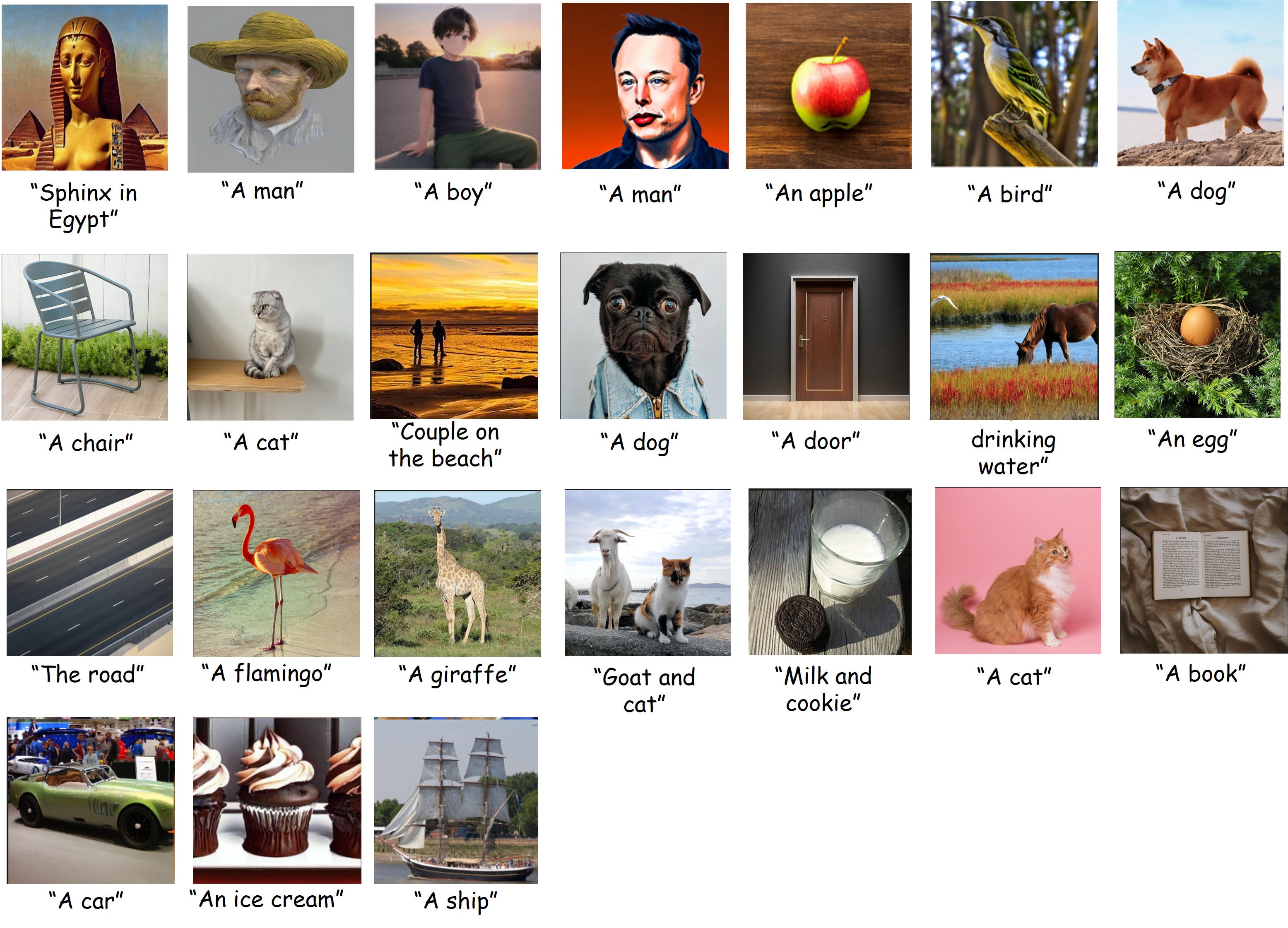}
    \vspace{-20pt}
    \caption{Images of the testing dataset (\textit{Part 4}).}
    \label{fig:data4}
    \vspace{-14pt}
\end{figure*}

 \begin{figure*}[t]
    \centering
    \renewcommand*{\thefigure}{S\arabic{figure}}
    \includegraphics[width=0.64\linewidth]{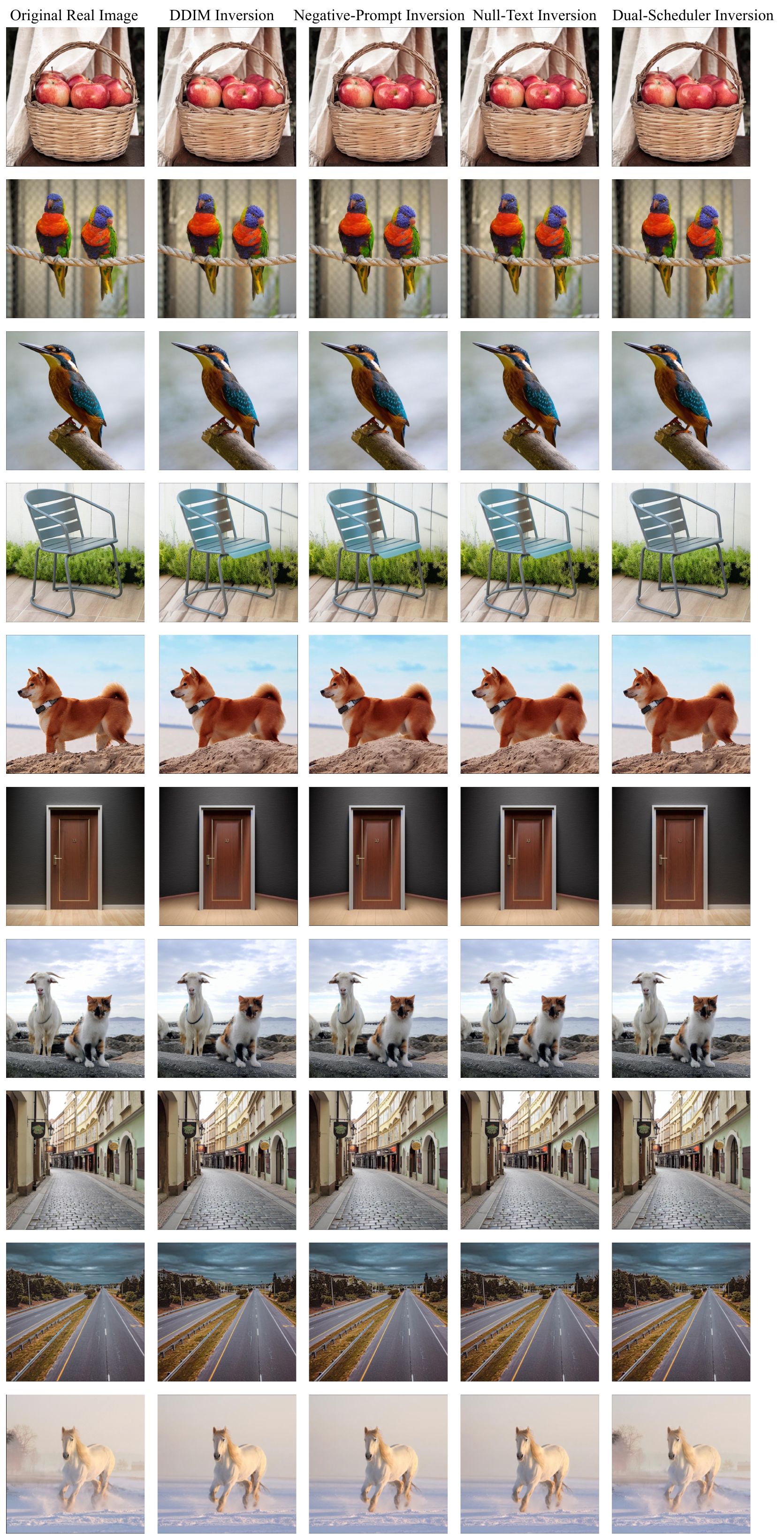}
    \caption{Reconstruction examples (\textit{guidance scale $w=1$}). While Null-Text Inversion requires fine-tuning, the other three methods do not. Dual-Schedule Inversion achieves excellent performance without fine-tuning.}
    \label{fig:rec1}
    \vspace{-14pt}
\end{figure*}

 \begin{figure*}[t]
    \centering
    \renewcommand*{\thefigure}{S\arabic{figure}}
    \includegraphics[width=0.64\linewidth]{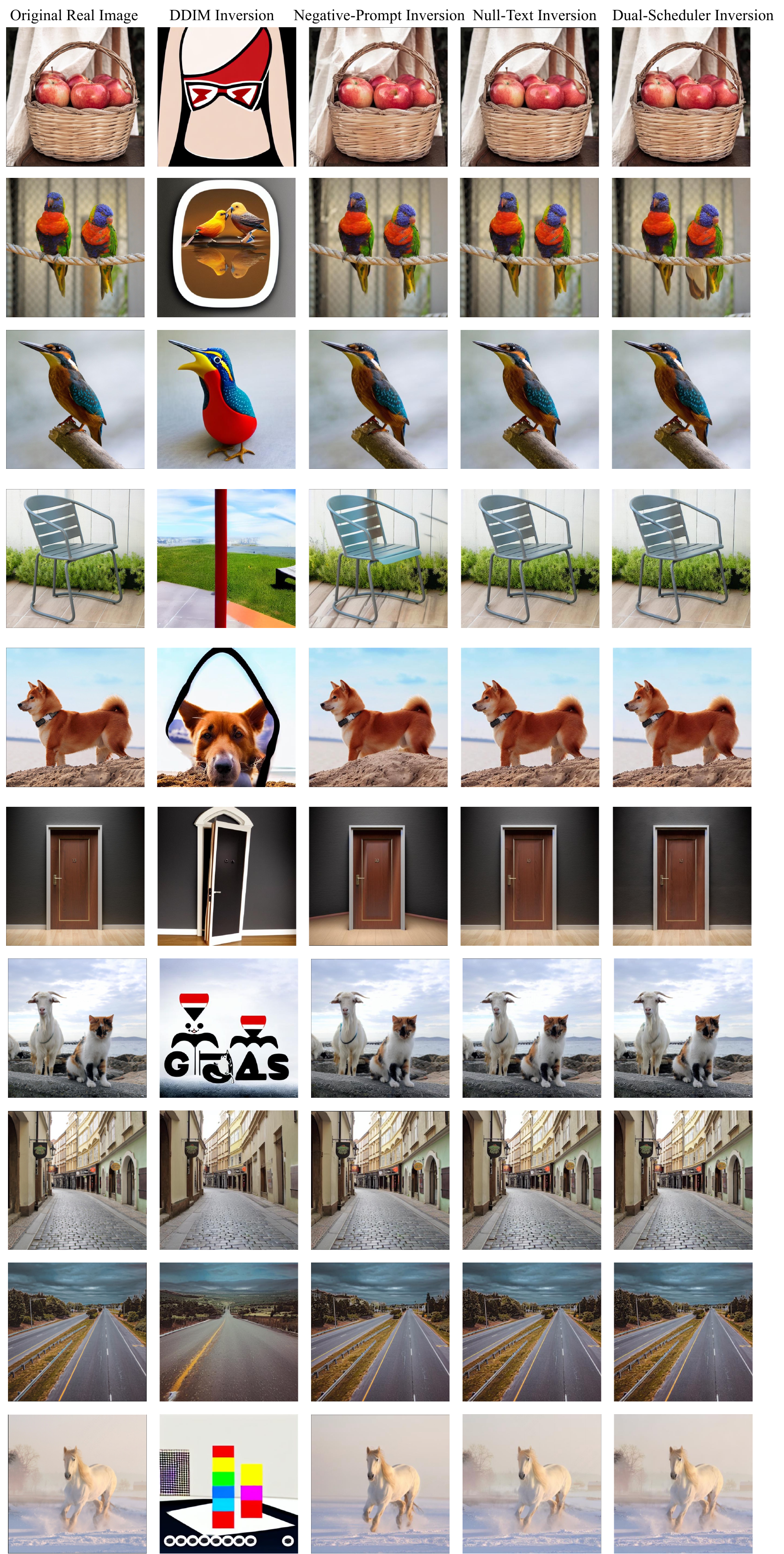}
    \caption{Reconstruction examples (\textit{guidance scale $w=4$}). While Null-Text Inversion requires fine-tuning, the other three methods do not. Dual-Schedule Inversion achieves excellent performance without fine-tuning.}
    \label{fig:rec2}
    \vspace{-14pt}
\end{figure*}

 \begin{figure*}[t]
    \centering
    \renewcommand*{\thefigure}{S\arabic{figure}}
    \includegraphics[width=0.64\linewidth]{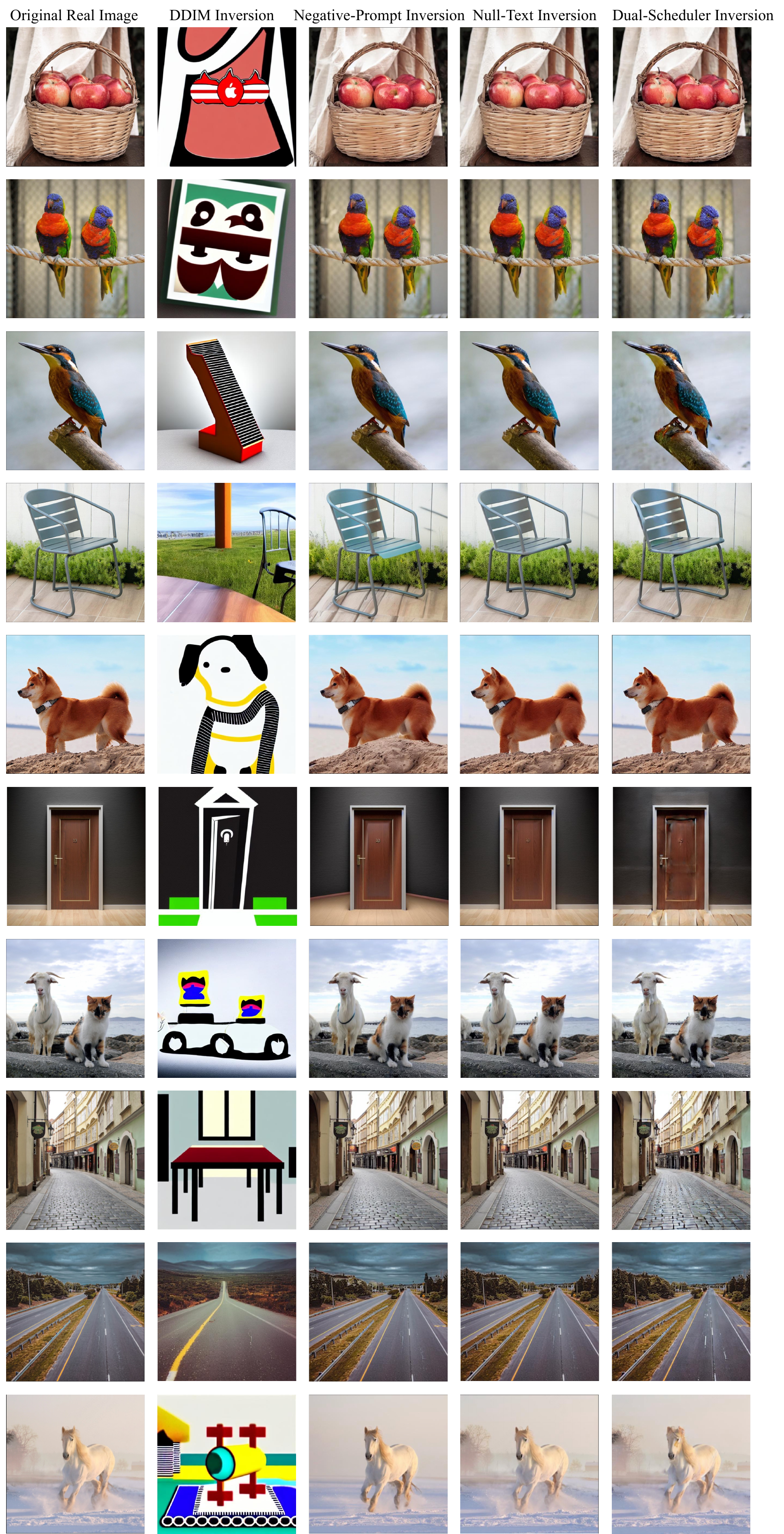}
    \caption{Reconstruction examples (\textit{guidance scale $w=7.5$}). While Null-Text Inversion requires fine-tuning, the other three methods do not. Dual-Schedule Inversion achieves excellent performance without fine-tuning.}
    \label{fig:rec3}
    \vspace{-14pt}
\end{figure*}

\end{document}